\documentclass[11pt]{article}

\usepackage{amsmath,amssymb,amsfonts,pifont}
\usepackage{multicol}
\usepackage{amstext}
\usepackage{amsthm}
\usepackage{multirow}
\usepackage{booktabs}
\usepackage[skip=0pt]{subcaption}
\usepackage{times}
\usepackage{lipsum}
\usepackage[shortlabels]{enumitem}
\usepackage{cancel}
\usepackage{wrapfig}
\usepackage{array}
\usepackage{siunitx}
\usepackage{csvsimple}
\usepackage[multidot]{grffile}
\usepackage{bbm}
\usepackage{dblfloatfix}
\usepackage[unicode,psdextra, backref=page]{hyperref}

\usepackage{makecell}
\usepackage{bbm, dsfont}
\usepackage{mathtools}
\usepackage{xcolor}
\usepackage{comment}
\usepackage{blkarray}
\hypersetup{colorlinks = true,linkcolor = blue,anchorcolor =red,citecolor = blue,filecolor = red,urlcolor = red, pdfauthor=author}

\usepackage{geometry}
\usepackage{latexsym}
\usepackage{palatino}
\usepackage{mathpazo}
\usepackage{enumitem}

\usepackage[multiple]{footmisc}
\usepackage{mathrsfs}
\usepackage{tikz}
\usepackage{cleveref}

\usepackage{algpseudocode,algorithm,algorithmicx}
\usepackage{xfrac}

\usepackage{graphicx} %
\usepackage{color}

\usepackage{array}
\usepackage{amssymb}
\usepackage{amsmath}
\usepackage{xspace}
\usepackage{fancyhdr}
\usepackage{comment}

\usepackage[numbers, sort, comma, square]{natbib}
\usepackage{fullpage}
	
\usepackage[colorinlistoftodos,prependcaption,textsize=tiny]{todonotes}

\renewcommand{\leq}{\leqslant}
\renewcommand{\geq}{\geqslant}
\renewcommand{\le}{\leqslant}
\renewcommand{\ge}{\geqslant}

\hypersetup{final}

\newcommand{\meansquared}{\mathsf{err}_{\ell_2^2}}
\newcommand{\absoluteerror}{\mathsf{err}_{\ell_\infty}}

\newcommand{\variance}{\sigma_{\epsilon,\delta}}

\newcommand{\calM}{\ensuremath{\mathcal{M}}}

\newcommand{\calX}{\ensuremath{\mathcal{X}}}

\newcommand{\E}{\mathop{\mathbf{E}}}

\newcommand{\R}{\mathbb{R}}
\newcommand{\real}{\mathbb{R}}
\newcommand{\complex}{\mathbb{C}}

\newtheorem{lem}{Lemma}

\newtheorem{remark}[lem]{Remark}

\newtheorem{defn}[lem]{Definition}

\newtheorem{claim}[lem]{Claim}
\newtheorem{definition}[lem]{Definition}

\makeatletter
\newcommand{\vast}{\bBigg@{4}}
\newcommand{\Vast}{\bBigg@{5}}
\makeatother

\newcommand{\ex}[2]{{\ifx&#1& \mathbb{E} \else
\underset{#1}{\mathbb{E}} \fi \left[#2\right]}}
\newcommand{\pr}[2]{{\ifx&#1& \mathbb{P} \else
\underset{#1}{\mathbb{P}} \fi \left[#2\right]}}

\DeclarePairedDelimiter\abs{\lvert}{\rvert}

\DeclarePairedDelimiterX{\infdivx}[2]{(}{)}{%
  #1\;\delimsize\|\;#2%
}

\renewcommand{\epsilon}{\varepsilon}

\newcommand{\tr}[1]{{\sf Tr}\left(#1\right)}

\newcommand{\set}[1]{\left\{ {#1} \right\}}
\newcommand{\norm}[1]{{\left\Vert {#1} \right\Vert}}
\newcommand{\paren}[1]{\left( {#1} \right)}
\newcommand{\sparen}[1]{\left[ {#1} \right]}

\setlist{nolistsep}
\setlist[itemize]{noitemsep, topsep=0pt}

\setlist{nolistsep}
\setlist[itemize]{noitemsep, topsep=0pt}

\newcommand{\op}[1]{\operatorname{#1}}

\newtheorem{theorem}[lem]{Theorem}

\def\I{\mathbb{1}}

\usepackage{wrapfig}

\makeatletter
\def\blfootnote{\xdef\@thefnmark{}\@footnotetext}
\makeatother

\newcommand{\exponentialbound}{\paren{1 + {1 \over \pi} \paren{ {\alpha^2 \over (\alpha^2-1)^2} -\frac{\alpha^2}{T (\alpha^2-1) \alpha^{2T}} }}}

\newcommand{\polynomialbound}{\paren{1 + {1 \over 4(2c-1)} - {(T+1)^{1-2c} \over 4(2c-1)}}}

\newcommand{\herdisc}{\mathsf{herdisc}}
\newcommand{\disc}{\mathsf{disc}}


\title{A Unifying Framework for Differentially Private Sums under Continual Observation}

\author
{
Monika Henzinger
\thanks{
Institute of Science and Technology, Austria.
See funding information in the acknowlegement section.
email: \texttt{monika.henzinger@ist.ac.at}}
\and 
Jalaj Upadhyay\thanks{Rutgers University. See funding information in the acknowlegement section. 
email: \texttt{jalaj.upadhyay@rutgers.edu}}
\and
Sarvagya Upadhyay\thanks{Fujitsu Research of America
email: \texttt{supadhyay@fujitsu.com} }
}

\begin{document}

\maketitle
\begin{abstract}    
    We study the problem of maintaining a differentially private decaying sum  under continual observation. We give a unifying framework and an efficient algorithm for this problem for \emph{any sufficiently smooth} function. Our algorithm is the first differentially private algorithm that does not have a multiplicative error for polynomially-decaying weights. Our algorithm improves on all prior works on differentially private decaying sums under continual observation and  recovers exactly the additive error for the special case of continual counting from Henzinger et al. (SODA 2023) as a corollary. 
    
    Our algorithm is a variant of the factorization mechanism whose error depends on the $\gamma_2$ and $\gamma_F$ norm of the underlying matrix. We give a constructive proof for an almost exact upper bound on the $\gamma_2$ and $\gamma_F$ norm and an almost tight lower bound on the $\gamma_2$ norm for a large class of lower-triangular matrices. This is the first non-trivial lower bound for lower-triangular matrices whose non-zero entries are not all the same. It includes matrices for all continual decaying sums problems, resulting in an  upper bound on the additive error of any differentially private decaying sums algorithm under continual observation.

    We also explore some implications of our result in discrepancy theory and operator algebra. Given the importance of the $\gamma_2$ norm in computer science and the extensive work in mathematics, we believe our result will have further applications.

\end{abstract}

\thispagestyle{empty}

\clearpage

\tableofcontents

\thispagestyle{empty}
\clearpage

\pagenumbering{arabic} 
\section{Introduction}
\label{sec:introduction}
When aggregating streams of data such as system variables or infection numbers,  decaying sums are used to reduce the impact of old data, where the specific decaying weight depends on the application. This is, for example, the case if the ``state'' of a system changes over time and, in order to correctly capture its current state, the most recent data should be given higher weight.
Popular decaying sums are sliding windows, where the last $W$ data points receive weight 1  and all others weight 0~\cite{datar2002maintaining}, or exponential~\cite{jacobson1988congestion} or polynomial weight functions~\cite{cohen2003maintaining}, where the weight of data streamed at $i$ time steps earlier decreases by a function that decays  exponentially, resp.~polynomially in $i$.
If the data contains sensitive information,  another desired requirement is that the ``noisy'' approximations to these sums also guarantee {\em differential privacy}. 

Differential privacy on data streams, called \emph{differential privacy under continual observation}, was introduced 
in seminal works by Dwork et al.~\cite{dwork2010differentially} and Chan et al.~\cite{chan2011private}. They gave upper and lower bounds on the additive error when all weights equal 1.  Since then, several works have studied the continual release model and its applications~\cite{andersson2023smooth, cardoso2021differentially, chan2012differentially, epasto2023differentially, 
fichtenberger2021differentially, ghazi2022differentially,  huang2021frequency, upadhyay2021differentially,upadhyay2021framework}, with some renewed recent interest due to its application in private optimization~\cite{choquette2022multi,mcmahan2022private,han2022private,henzinger2022almost,kairouz2021practical,koloskova2023convergence}. The setting in \cite{chan2011private, dwork2010differentially} was generalized in an elegant work by Bolot et al.~\cite{bolot2013private}, who studied differentially private decaying sums under continual observation for polynomial decay, exponential decay, and the sliding window model. 
While they gave algorithms with a polylogarithmic additive error bound, they suffer from  various limitations which  limit their usage in applications, including the two main motivations (estimating infectious disease spread~\cite{dwork2010differentially, henzinger2022constant} and private online optimization~\cite{kairouz2021practical})  behind the recent interest in private continual observation: 
\begin{enumerate}
    \item These algorithms only provide asymptotic bounds that suffer from the inherent limitations discussed in several works~\cite{andersson2023smooth,choquette2022multi, mcmahan2022private, henzinger2022constant, henzinger2022almost}: the error is not the best possible in terms of constants and is non-smooth.

    \item Every class of functions has a tailor-made algorithm increasing the engineering effort. 

    \item The algorithm for polynomial decay has both an additive \emph{and} a multiplicative error and the algorithm for exponential decay, $f(n) =\alpha^{-n}$, does not provide any guarantee for $\alpha>3/2$.
\end{enumerate}

We resolve all these issues. We give a general framework for maintaining differentially private decaying sums under continual observation for  sufficiently smooth decaying weight function based on the factorization mechanism~\cite{edmonds2020power, li2010optimizing}. 
We show how to reduce finding an algorithm for a given weight function to finding the associated symbol of a suitable Toeplitz operator.
We show how to do so for all sufficiently smooth functions, without multiplicative error, and give a bound on the resulting additive error, both in $\ell_\infty$ and $\ell_2$-norm.
This, in particular, implies a novel algorithm with a small additive error for the
three decaying functions mentioned above. 

\subsection{Main Result}
\label{sec:mainresult}
We state our main result in terms of two {\em factorization norms} ($\gamma_2$ and $\gamma_F$) of a suitable matrix representing the decaying sum. 
These matrix norms 
elegantly characterize the error of answering linear queries under differential privacy, and have also appeared naturally in many areas of mathematics and computer science~\cite{aharonov1998quantum, lee2008direct, matouvsek2020factorization,   pisier1986factorization,wald1994quantum}. 

\emph{Our main result is an almost exact bound on the $\gamma_2$ and $\gamma_F$ norm of 
a large class of lower-triangular matrices with entries defined by a monotonically 
non-increasing function.} 
The $\gamma_2$-norm of lower-triangular matrices has a special place in operator algebra~\cite{aleksandrov2023triangular,birman2012spectral, drmavc1994perturbation, gluskin2019symplectic,gohberg1970theory, kato1973continuity} since  Kwapien and Pe{\l}czy{\'n}ski~\cite{kwapien1970main} studied it to answer two fundamental questions: Problem 88 by Mazur in
the Scottish Book~\cite{mauldin2015scottish} and the absolutely summing problem of~\cite{lindenstrauss1968absolutely}.  To get our bounds, we compute all the coefficients of the square root of a formal power series related to these matrices. These power series arise naturally in many areas of physics and mathematics;  therefore, we believe, our results are of independent interest in these areas as well. 
While our primary focus is differential privacy, we discuss some other implications in \Cref{sec:operator}. 

To begin describing our results in more detail, we first define Bell's polynomial. Our tightest bound can be written in a closed-form formula 
using it.

\color{black}
\begin{defn}
[Bell's polynomial~\cite{bell1934exponential}]
For $k,n \in \mathbb N$ with $k\leq n$ and a sequence $s_1, s_2, \cdots$, the Bell's polynomial, denoted  $B_{n,k}$, is defined  as follows:
\begin{align}
    \label{eq:bellpartialpolynomial}
B_{n,k}(s_1,s_2,\dots,) =
\sum {n! \over j_1!j_2!\cdots j_{n-k+1}!} \left({s_1\over 1!}\right)^{j_1}\left({s_2\over 2!}\right)^{j_2}\cdots\left({s_{n-k+1} \over (n-k+1)!}\right)^{j_{n-k+1}},
\end{align}
 where the summation is over all integral $j_1, \cdots, j_{n-k+1}$ such that 
$
j_1 + \cdots j_{n-k+1}=k$ and $ 
j_1 +2j_2 + \cdots + (n-k+1)j_{n-k+1} =n.$
\end{defn}

\medskip \noindent \textbf{Factorization norm of a family of lower-triangular matrices.} For a matrix $A \in \real^{m \times n}$ and $p,q\in \mathbb N$, let $\norm{A}_{p\rightarrow q} = \min_{\norm{x}_p=1} \norm{Ax}_q$. The $\gamma_2$ and $\gamma_F$ norms of $A$ are defined as 
\begin{align*}
\gamma_2(A) = \min_{LR=A} \set{\norm{L}_{2 \to \infty} \norm{R}_{1 \to 2} }
 \qquad \text{and} \qquad  
\gamma_F(A) = \min_{LR=A} \set{\norm{L}_{F} \norm{R}_{1 \to 2}}.
\end{align*} 
Let $\mathbb R_+$ and $\mathbb N_+$ be the sets of positive real numbers and positive natural numbers, respectively. Assume that $f : \mathbb N_+ \to \mathbb R_+$ is a monotonically non-increasing function with $f(1)=1$. We consider the following family of Toeplitz matrices:
\begin{align}
M_f = \begin{pmatrix}
f(1) & 0 & \cdots & 0 \\
f(2) & f(1) & \cdots & 0 \\
\vdots & \vdots & \ddots & \vdots \\
f(T) & f(T-1) & \cdots & f(1) 
\end{pmatrix} \in \mathbb R_+^{T \times T}.
\label{def:Mf}
\end{align}
We show the following upper and lower bounds on these norms for  the matrix $M_f$.
\begin{theorem}
[Bounds on factorization norm.]
\label{thm:gamma2norm}
    Let $M_f$ be the matrix as defined in \cref{def:Mf} and parameterized by a monotonically non-increasing function $f : \mathbb N_+ \to \mathbb R_+$ such that $f(1)=1$. Let
    \begin{align}
            s_n =  {n!  f(n+1)} \quad \text{and} \quad a_n= {1\over n!} \sum_{k=1}^n B_{n,k} \paren{s_1, s_2, \cdots} \prod_{m=0}^{k-1}\paren{{1\over 2}-m}.
    \label{eq:sn}
    \end{align}
    Then {$1 < {2 \over \sqrt{4 -f(2)^2}} \leq \gamma_2(M_f) \leq 1 + \sum_{n=1}^{T-1} a_n^2$} and $\gamma_F(M_f) \leq \sqrt{T}\gamma_2(M_f)$. 
\end{theorem}
Even though $a_n$ depends on both $n$ and $f$, we use $a_n$ and not $a_{n,f}$ to simplify the notation. Our lower bound on $\gamma_2(M_f)$ is the first non-trivial improvement on the trivial bound $\gamma_2(M_f) \geq 1$ for a non-constant function $f$. 

Bell's polynomial has  played a central role in many areas of mathematics, including enumerative combinatorics, analysis, and umbral calculus (also see~\cite{o2022moivre}). To the best of our knowledge, this is the first application of Bell's polynomial in  differential privacy.  Bell polynomials have a nice recurrence relation~\cite{comtet1974advanced} and, thus, can be efficiently evaluated using dynamic programming. However, we also provide an easy-to-state upper bound for a large class of matrices so that dynamic programming is not required. In particular, consider the family of functions 
\[
\mathcal F = \set{f: \mathbb N_+ \to \mathbb R_+ : \text{the  coefficients of the formal power series of } \sqrt{1 + \sum_{i \geq 1}f(i+1)x^i} \text{ are in } \mathbb R_+}.
\]
Several interesting classes of functions belong to $\mathcal F$. Some classic examples 
are as follows:
\begin{enumerate}
    \item $f(n)=\alpha^{-n}$ (for a constant $\alpha \geq 1$). In this case, $M_f$ is the matrix that encodes the \emph{exponentially decaying weight functions} and  $1 + \sum_{i \geq 1}f(i+1)x^i = (1 - x/\alpha)^{-1}$.  In control theory, this is known as {\em rational transfer function}. Note that for $\alpha = 1$, this is simply the constant function $f(n) = 1$ studied in Chan et al.~\cite{chan2011private} and Dwork et al.~\cite{dwork2010differentially}.
    \label{item:exponential}
    
    \item $f(n)=n^{-c}$ (for $c \in \mathbb N_+$). In this case, $M_f$ is the matrix that encodes the \emph{polynomially decaying weight functions} and $1 + \sum_{i \geq 1}f(i+1)x^i={\mathsf{Li}_c(x) \over x}$, where $\mathsf{Li}_c(x)$ is the Jonquière's or polylogarithmic function that appears in Fermi–Dirac integral and processes involving higher order Feynmann diagram.
    \label{item:polylog}
\end{enumerate}
 
For $f \in \mathcal F$, we give the following upper bound:
\begin{theorem}
\label{cor:gamma2norm}
\label{cor:gammaFnorm}
    Let $f\in \mathcal F$ and $M_f$ be the $T \times T$ matrix defined in \cref{def:Mf}. Then 
    \begin{align}
        \gamma_2(M_f) \leq 1 + \sum_{n=1}^{T-1} {f(n+1)^2 \over 4}\quad  \text{and} \quad 
    \gamma_F(M_f) \leq \sqrt{T} \paren{1 + \sum_{n=1}^{T-1} {f(n+1)^2 \over 4}}.
        \label{eq:gamma2norm}
    \end{align}
In particular, when  $f(n)=n^{-c}$ for  $c \in \mathbb N_+$, then 
    $\gamma_2(M_f) \leq 1 + {H_{T,2c}-1 \over 4} \leq 1 + {\zeta(2c)-1 \over 4}$, where $H_{T,2c} = \sum_{n=1}^{T-1} {1 \over n^{2c}}$ is the generalized Harmonic sum and $\zeta(2c)$ is the Riemann zeta function of order $2c$. 
\end{theorem}

We show formally in \Cref{lem:comparisongammanorm} that \Cref{cor:gamma2norm} improves the previously best known upper bounds on  $\gamma_2(M_f)$ for $f(n)=n^{-c}$ for all values of $c$ and $T \geq 2$. 
\Cref{cor:gamma2norm} relies on proving a tight upper bound on $a_n$ in \Cref{thm:gamma2norm}. In particular, for $c=1$, the gap between our upper bound on $a_{2048}$ in \Cref{cor:gamma2norm} and $a_{2048}$ in \Cref{thm:gamma2norm} is  $\approx 6 \times 10^{-8}$ and for  $c=2$, the gap reduces to $\approx 1.4 \times 10^{-14}$ (also see \Cref{sec:figures}).

Our bound also converges to the exact value as $c$ increases: for all values of $T \in \mathbb N_+$, the additive gap  between our upper and lower bound is at most $0.13$ (for $c=1$), at most $0.0125$ (for $c=2$), at most $0.003$ (for $c=3$), at most $5 \times 10^{-4}$ (for $c=4$), and at most $1.25 \times 10^{-4}$ (for $c=5$). 

A proof of \cref{eq:gamma2norm} is presented in \Cref{sec:proofmaintheorem} and the case of $f(n)=n^{-c}$ is in \Cref{sec:specialmatrices}.
\Cref{cor:gamma2norm} can be used to get a bound when $f(n)=\alpha^{-n}$, but one can get a tighter bound using the exact expression of $a_n$ in \Cref{thm:gamma2norm} and bounds on the evaluation of Bell's polynomial on special inputs, showing the versatility of our approach. We show the following in \Cref{sec:exponential}:   

\begin{theorem}
\label{thm:exponential}
Let $M_f$ be the matrix defined by the function $f(n)=\alpha^{-n}$. When $\alpha=1$, then we recover the factorization norm bounds in~\cite{henzinger2022constant, henzinger2022almost}. Let $S_{T,2\alpha}= \sum_{n=1}^{T-1} {1 \over n\alpha^{2n}}$. If $\alpha >  1$, then  
\begin{align*}
\gamma_2(M_f) & \leq 1 + {1 \over \pi}S_{T,2\alpha} \leq \exponentialbound, \quad \text{and} \\
\gamma_F(M_f) &\leq \sqrt{T}\paren{1 + {1 \over \pi}S_{T,2\alpha}} \leq \sqrt T\exponentialbound.    
\end{align*}
\end{theorem}

\subsection{Applications
in Differential Privacy}
\label{sec:applicationdp}
We now apply our bounds on the factorization norm to differential privacy under continual observation (see \Cref{defn:dp}). Given a 
function $f: \mathbb N_+ \to \mathbb R_+$ and a stream $x_1, x_2, \dots, x_T$ of real numbers, arriving one per time step, our goal is to output at each time step $1 \le t \le T$ the value $\sum_{i=1}^t x_i f(t-i+1)$ in a differentially private manner. This is called the \emph{continual decaying sums problem (CDS problem)}.
Computing all $T$ values (exactly) simply corresponds to 
computing $M_f x$, where $M_f$ is as defined in \cref{def:Mf} and $x$ is the $T$-dimensional vector formed by $x_1, x_2, \dots, x_T$.
Furthermore,  the decaying sum at time step $t$ can be  computed either by using the $t \times t$-dimensional submatrix of $M_f$ consisting of the $t$ first rows and columns, or by using a vector $x' \in \R^T$, where 
$x'[i] = x_i$ for $i \le t$ and $x'[i] = 0$ for $i>t$.

In the static, i.e. non-continual, setting, algorithms for computing the product of a public matrix $A$ and a privately-given vector $x$ are well studied and their quality is usually measured by the {\em (additive) mean-squared error} (aka $\ell_2^2$-error) and the and the \emph{(additive) absolute error} (aka $\ell_\infty$-error). The (additive) mean-squared error of a randomized algorithm $\mathsf{M}$  for computing $A x$ on any real input vector $x$ is defined as 
\begin{align}
\meansquared(\mathsf M,A, T) = \max_{ x \in \real^T} \E_{\mathsf M} \sparen{ \frac{1}{T } \norm{\mathcal{M}(x) - A x}_2^2}
\end{align}
and the (additive) absolute error is defined as
\begin{align}
\absoluteerror(\mathsf M,A, T)
= \max_{ x \in \mathbb R^T} \E \sparen{\norm{\mathsf{M}(x) - A x}_{\infty}}.
\end{align}

One popular $(\epsilon, \delta)$-differentially 
private algorithm  for this problem is the 
\emph{factorization mechanism}~\cite{edmonds2020power,li2010optimizing} $\mathsf M_{L,R}$ 
that, given a factorization of $A$ into two matrices $L$ and $R$, i.e., $A = LR$, outputs $L(Rx + z)$, \emph{where $z$ is a suitable noise vector that is independent of the values of the input $x$}. 
This property is crucial to allow us to use the factorization mechanism for the CDS problem in the continual observation setting (\Cref{alg:factorizationmechanism}).

Given a bound $\Delta >0$ and $x \in [-\Delta, \Delta]^T$
Li et al.~\cite{li2010optimizing} showed that, for an optimal choice of $L$ and $R$, the mean-squared error is
\begin{align}
\label{eq:meansquaredgammanorm}
\meansquared(\mathsf M_{L,R}, A,T) = {\frac{1}T }\variance^2 \Delta^2 \gamma_{\op{F}}(A)^2, \quad \text{where} \quad \variance = {2 \sqrt{2 \log(1.25/\delta)} \over \epsilon},
\end{align}
and Edmonds et al.~\cite{edmonds2020power} showed that  the absolute error is 
\begin{align}
\absoluteerror(\mathsf M_{L,R}, A,T) \leq  \variance \Delta \gamma_2(A) \sqrt{\log T}.
\label{eq:absoluteerrorgammanorm}   \end{align}

\noindent By ensuring lower-triangular $L$ and $R$, we use  \cref{eq:meansquaredgammanorm}, 
\cref{eq:absoluteerrorgammanorm}, and   \Cref{thm:gamma2norm} to show the following:
\begin{theorem}
\label{thm:privatecounting}
    Let $M_f$ be the matrix as defined in \cref{def:Mf} and parameterized by a monotonically non-increasing function $f : \mathbb N_+ \to \mathbb R_+$ such that $f(1)=1$. Let $a_n$ be as in \cref{eq:sn}.
    Then \Cref{alg:factorizationmechanism} is an efficient $(\epsilon,\delta)$-differentially private algorithm for the CDS problem corresponding to the function $f:\mathbb N_+ \to \mathbb R_+$, such that, {simultaneously for all $1 \leq t\leq T$,} 
    \[
    \absoluteerror(\mathsf M_{L,R},M_f,T) \leq \variance \Delta \paren{1 + \sum_{n=1}^T a_n^2}\sqrt{\log(T)}, \quad  
    \meansquared(\mathsf M_{L,R},M_f,T) \leq \variance^2 \Delta^2 \paren{1 + \sum_{n=1}^T a_n^2}^2.
    \]
\end{theorem}

\begin{algorithm}[t]
\caption{General Mechanism for CDS problem, $\calM_{\op{fact}}$, for a function $f:\mathbb N_+ \to \mathbb R_{+}$}
\begin{algorithmic}[1]
   \Require A monotonically decreasing function $f:\mathbb N_+ \to \mathbb R_{+}$ and a stream of bits $(x_1,\cdots, x_T)$, length of the stream $T$, $(\epsilon,\delta)$: privacy budget.
   \State Define $M_f$ as in \cref{def:Mf} and $r(i)=a_{i-1}$ in \Cref{thm:gamma2norm} ($i \geq 2$) with $r(1)=1$ or by solving $T$ linear equations by computing the values of $r(1), r(2), \cdots, r(T)$ in order (\Cref{rem:coefficient}). Define
    \[
    L:= \begin{pmatrix}
        r(1) & 0 & \cdots & 0 \\
        r(2) & r(1) & \cdots & 0 \\
    \vdots & \vdots & \ddots & \vdots \\
        r(T) & r(T-1) & \cdots & r(1)  
    \end{pmatrix}.
    \]

    \State Use the factorization mechanism to output a differentially private decaying sum:
    \begin{itemize}
        \item Sample a  random vector $z \sim \mathcal N(0, \sigma^2 LL^\top)$ for $\sigma^2 = \variance^2 \Delta^2 \sum_{i=1}^T r(i)^2 = \variance^2 \Delta^2 \norm{L}_{1 \to 2}^2$. 
        \item On receiving $x_t$ at time step $t$, define $x' = \begin{pmatrix}
    x_1 & \cdots & x_t & 0 & \cdots & 0 \end{pmatrix}^T \in [-\Delta,\Delta]^T$ formed by the first $t$ updates.  
        \item Output $L_{[t:]}L x' + z[t]$ at time step $t$, where $L_{[t:]}$ be the $t$-th row of $L$.
    \end{itemize}
\end{algorithmic}
\label{alg:factorizationmechanism}
\end{algorithm}

\noindent \Cref{thm:privatecounting} implies the following results for specific functions (see \Cref{sec:privatecounting} for a proof).
\begin{theorem}
\label{cor:privatecounting}
     When $f \in \mathcal F$, \Cref{alg:factorizationmechanism} is an efficient $(\epsilon,\delta)$-differentially private algorithm that solves the CDS problem corresponding to the function $f$ such that for all $1\leq t \leq T$, $\meansquared(\mathsf M_{L,R},M_f,T)  \leq \variance^2 \Delta^2 \paren{1 + \sum_{n=1}^T \frac{f(n+1)^2}{4} }^2$ and 
    $\absoluteerror(\mathsf M_{L,R},M_f,T)  \leq \variance \Delta \paren{1 + \sum_{n=1}^T \frac{f(n+1)^2}{4} }\sqrt{\log(T)}$. 
 In particular, {  for all $ 1 \leq t\leq T$,} 
 \begin{enumerate}
     \item when $f(n)=n^{-c}$ for $c\in \mathbb N_+$, then 
     $\absoluteerror(\mathsf M_{L,R},M_f,T) \leq \variance \Delta \paren{1 + {H_{T,2c}-1 \over 4}} \sqrt{\log(T)}$ and 
     $\meansquared(\mathsf M_{L,R},M_f,T) \leq \variance^2 \Delta^2 \paren{1 + {H_{T,2c}-1 \over 4}}^2$, where $H_{T,2c}$ is the generalized Harmonic sum; and

     \item when $f(n)= \alpha^{-n}$ for $\alpha > 1$, then $\absoluteerror(\mathsf M_{L,R},M_f,T) \leq \variance \Delta (1 + {1 \over \pi}S_{T,2\alpha})\sqrt{\log(T)}$ and  $\meansquared(\mathsf M_{L,R},M_f,T) \leq \variance^2 \Delta^2 \paren{1 + {1 \over \pi}S_{T,2\alpha}}^2,$ where $S_{T,2\alpha} \leq {\alpha^2 \over (\alpha^2-1)^2} - {\alpha^2 \over T(\alpha^2-1)\alpha^{2T}}$. 
 \end{enumerate}
\end{theorem}

\noindent \Cref{cor:privatecounting} allows us to extend all the applications of private continual counting~\cite{cardoso2021differentially, epasto2023differentially,fichtenberger2021differentially, kairouz2021practical, ghazi2022differentially, smith2017interaction} with $f(n)=1$ to the setting of decaying weights on the data.

We also give an algorithm for the sliding window model in \Cref{sec:slidingwindow} that achieves a tight accuracy guarantee with respect to constants. 
\begin{theorem}
\label{thm:slidingwindowmodel}
For $w \in \mathbb N_+$, let $f:\mathbb N_+ \to \{0,1\}$ be the function such that $f(n) = 1$ if and only if $n \leq w$. Then, for all $1 \leq t \leq T$, $\meansquared(\mathsf M_{L,R},M_f,T) \leq 2\variance^2 \Delta^2 \paren{1 + {\log(w) \over \pi} + {2\over w} }^2$ and $\absoluteerror(\mathsf M_{L,R},M_f,T) \leq \variance \Delta \paren{1 + {\log(w) \over \pi} + {2\over w} } \sqrt{2\log(T)}$.
\end{theorem}

{
 \noindent \textbf{Comparison with prior work.}
The state-of-the-art on the CDS problem is as follows:

(A) For $(\epsilon,\delta)$-differential privacy, at each time step, the \emph{Gaussian mechanism} (see \Cref{def:gaussianmechanism}) adds noise based on the sensitivity of the sum.
We show in \Cref{sec:gaussiancomparison} that our error for the functions in \Cref{cor:privatecounting} is less than the corresponding errors of the Gaussian mechanism.

(B) 
To the best of our knowledge, the only prior work that studied the CDS problem for some specific non-constant decaying function is the work by Bolot et al.~\cite{bolot2013private}. They studied three classes of functions, namely sliding window, exponential decay, and polynomial decay, and only analyzed the $\ell_\infty$-error under $\epsilon$-DP. We next compare our bounds for all three classes of decay functions. To do so, we consider $(\epsilon,\delta)$-differentially private variants of their algorithm by replacing Laplacian noise with Gaussian noise in their algorithms.

(B.1) 
For $f(n)=n^{-c}$ for $c \in \mathbb N_+$, and an input stream of $T$ bits $x_i$, {i.e., $\Delta = 1$}, let $F_p(c,t) = \sum_{i=1}^t x_i \paren{\frac{1}{t-i+1}}^c$ denote the true sum. 
Bolot et al.~\cite{bolot2013private} showed that there exists a
differentially private
algorithm that, simultaneously for all $t \leq T$ and for $\beta \in (0,1)$, outputs $\widehat F_p(c,t)$ such that 
\begin{align}
\vert \widehat F_p(c,t) - F_p(c,t)\vert \leq  \beta F_p(c,t) + O \paren{ \frac{\variance}{c^{3/2}\beta^3} \log\paren{\frac{1}{1-\beta}} \sqrt{\log(T)} }.
\end{align}
Note that there is a multiplicative error scaling with $(1+\beta)$ \emph{as well as} an additive error.
Having a multiplicative error is highly undesirable in many applications where continual decaying sums are used 
as  the resulting  signal-to-noise ratio might be too low for any meaningful interpretation of the result. 
Our algorithm does not have a multiplicative error.  We give a detailed comparison in \Cref{sec:bolot} for their best-case scenario (i.e. when $F_p(c,t)=0$ or stream is  all-zero). 

(B.2) For $f(n)=\alpha^{-n}$, Bolot et al.~\cite{bolot2013private} only analyze their algorithm for $\alpha \in (1,3/2)$. In contrast, we give bounds for all $\alpha \geq 1$. Further, since their algorithm is a variant of the binary tree mechanism, even for $\alpha \in (1,3/2)$, our algorithm  improves the $\ell_\infty$-error as well as the $\ell_2^2$-error in the terms of constants as in the case of $f(n)=1$ (also see item (C) below).

(B.1) In the  sliding window model, their algorithm is based on the binary mechanism and achieves the same asymptotic errors as our algorithm. They  bound the error for any given time step by $O(\log(w) \sqrt{\log(1/\gamma)})$. To get a bound for all time steps simultaneously as we do, one has to set $\gamma=1/T$. The same argument as given for binary counting in~\cite{henzinger2022almost,henzinger2022constant} shows that we improve the constant factor in the $\ell_\infty$ and $\ell_2^2$-error by a factor of roughly 4 and 10, respectively.
}

(C) Continual counting studied by Chan et al.~\cite{chan2011private} and Dwork et al.~\cite{dwork2010differentially} is the CDS problem with the constant function $f(n) =1$.  Combined with \Cref{thm:exponential} and \cref{eq:meansquaredgammanorm} and \cref{eq:absoluteerrorgammanorm}, 
\Cref{thm:privatecounting} recovers as a special case exactly the two recent results~\cite{henzinger2022constant,henzinger2022almost} (also see \Cref{sec:exponential}). 
As noted in \cite[Remark 1.3]{henzinger2022constant} and \cite[Theorem 1.4]{henzinger2022almost}, this implies that our algorithm also improves the binary tree mechanism by a constant factor.

\section{Useful Preliminaries and Results}
Let $k$ be a non-negative integer. A function $f:\mathcal X \to \mathcal R$ is said to be of {\em differentiability class} $C^k$ if the first $k$ derivatives  exist and are continuous in $\mathcal X$. For a univariate function $f\in C^k$, we use the notation $f^{(k)}(x)$ to denote ${\mathsf d^k \over \mathsf dx^k} f(x)$ and $f^{(k)}(x)|_{x=a}$ to denote the evaluation of its $k$-th derivative at $x=a$. A {\em formal series} is an infinite sum. A {\em formal power series} is a special formal series, whose terms are of the form $a_n x^n$ for $n \in \mathbb N$. We use the following classical result named after Faà di Bruno and  first appeared in the calculus book of Arbogast~\cite{arbogast1800calcul} published in 1800.

\begin{theorem}
[Faà di Bruno's formula~\cite{arbogast1800calcul, faa1855sullo}]
\label{thm:faadibruno}
Let $n \in \mathbb N$. For any functions $F:\mathbb R \to \mathbb R$ and $G: \mathbb R \to \mathbb R$ that are in $C^n$, we have 
\[
{\mathsf d^n \over \mathsf dx^n} F(G(x)) = \sum_{k=1}^n F^{(k)}(G(x))\cdot B_{n,k}\left(G^{(1)}(x),G^{(2)}(x),\dots, G^{(n-k+1)}(x)\right).
\]
\end{theorem}

\subsection{Combinatorics.}  
Charalambides~\cite{charalambides2002enumerative} showed the following inversion formula, a shorter proof using \Cref{thm:faadibruno} can be found in Chou et al.~\cite{chou2006application}. 
\begin{theorem}
[Inversion formula~\cite{charalambides2002enumerative, chou2006application}] \label{thm:inversebellpolynomial}
Let $s_1, s_2, \cdots$ be a sequence of real numbers
and let $F:\mathbb R \to \mathbb R$ be a function such that there exists an $a \in \mathbb R$ with $F(a) \neq 0$ and such that $F(a+x)$ 
has a formal power series expansion in $x$. Let $G$ be its compositional inverse, i.e., $F(G(x)) = G(F(x))$ for all $x$ in the domain of $F$ and $G$. Then the following inversion result holds:
\begin{align}
    \begin{split}
        s_n &= \sum_{k=1}^n G^{(k)}(x) \vert_{x=f(a)} B_{n,k}(y_1, y_2, \cdots), \text{ and }
        y_n = \sum_{k=1}^n F^{(k)}(x) \vert_{x=a} B_{n,k}(s_1, s_2, \cdots)
    \end{split}
\end{align}
\end{theorem}

\subsection{Differential privacy}
The privacy definition we use in this paper is {\em differential privacy}. We define it next based on the notion of \emph{neighborhood} which we define below for different applications.
\begin{definition}
[Differential privacy~\cite{dwork2006calibrating}]
\label{defn:dp}
Let $\mathsf M : X \rightarrow R$  be a randomized algorithm mapping from a domain $X$ to a range $R$. $\mathsf{M}$ is $(\epsilon,\delta)$-differentially private if for every all neighboring dataset $D$ and $D'$ and every measurable set $S \subseteq R$,
$\mathsf{Pr} [\mathsf{M} (D) \in S] \leq e^\epsilon \mathsf{Pr} [\mathsf{M} (D') \in  S] + \delta.$
\end{definition}
Central to the notion of privacy is the notion of neighboring datasets. We use the standard notion of neighboring datasets.  For continual observation,  two streams, $S = (x_1,\cdots, x_T) \in \real^T$ and $S' = (x_1',\cdots, x_T') \in \mathbb R^T$ are neighboring if there is at most one $1 \leq i \leq T$ such that $x_i \neq x_i'$. This is known as {\em event level privacy}~\cite{ dwork2010differentially,chan2011private}.

One of the most common mechanisms to preserve differential privacy is the {\em Gaussian mechanism} (\Cref{def:gaussianmechanism}) that depends on the $\ell_2$ sensitivity of the function of interest, $F:\mathcal D \to \real^d$ defined over the domain $\mathcal D$. Formally, \emph{$\ell_2$ sensitivity} of a function $F: \mathcal D \to \real^d$ is the smallest number $L_F$ such that, for all $x,x' \in \mathcal D$ that differ in a single point, $\norm{F(x)-F(x')} \leq L_F$. Then Gaussian mechanism perturbs the output of the function of interest with a Gaussian noise that scales with the $\ell_2$-{\em sensitivity} of $F(\cdot)$.

\subsection{Toeplitz operator} More discussion on Toeplitz operator is presented in \Cref{sec:toeplitz}. Here, we just state basic results used in \Cref{sec:proofmaintheorem}. A $T \times T$ Toeplitz matrix, $A$ is a matrix such that $A[i,j]=a_{i-j}$ for some fixed $(a_{1-T}, a_{2-T}, \cdots, a_{T-2}, a_{T-1})$. A semi-infinite matrix, $\mathcal A$, of the same form is called a {\em Toeplitz} operator. For a Toeplitz operator $\mathcal A$ with $(i,j)$-th entry, $A[i,j]=a_{i-j}$, the {\em associated symbol} is $a(x)=\sum_k a_k x^k$. One important fact of the associated symbol is as follows:
\begin{theorem}
[\cite{bottcher2000toeplitz,conway2000course, trefethen2005spectra}]
\label{thm:product}
    If $a(x)$ is the associated symbol of Toeplitz operator $\mathcal A$ and $b(x)$ is that of $\mathcal B$, then $a(x)b(x)$ is the associated symbol of $\mathcal A \mathcal B$ and $\mathcal B \mathcal A$.
\end{theorem}

\section{Proofs of Results in \Cref{sec:mainresult} and \Cref{sec:applicationdp}} 
\label{sec:proofmaintheorem}
Let $f:\mathbb N_+ \to \mathbb R_+$ be a monotonically non-increasing function such that $f(1)=1$ and let $\mathcal M_f$ be its corresponding matrix for computing a decaying sum with the decaying function $f$.

\subsection{Proof of \Cref{thm:gamma2norm}}
For our lower bound, we use the dual characterization of $\gamma_2$ norm (\Cref{thm:haagerup}).  Let $\norm{\cdot}$ denote the spectral norm and  let $A \bullet Q$ denote the Hadamard (or Schur) product of matrices $A$ and $Q$. 
Let $M_f \in \real^{2 \times 2}$ be a $2 \times 2$ matrix as defined in \cref{def:Mf}. Consider the following unitary matrix: 
\[
Q = \begin{pmatrix}
    \sin(\theta) & \cos(\theta) \\
    -\cos(\theta) & \sin(\theta)
\end{pmatrix}, \quad \text{where} \quad \sin(\theta)= {\sqrt{4-2f(2)^2 \over 4- f(2)^2}}.
\]

\noindent Using the dual characterization of $\gamma_2$ norm (\Cref{thm:haagerup}) and \Cref{claim:norm}  in \Cref{sec:auxiliary}, we have  
\[
\gamma_2(M_f) = \max_{\norm{Q}=1} \norm{M_f \bullet Q} \geq \norm{M_f \bullet Q} = {2 \over \sqrt{4 -f(2)^2}} >1.
\] 
The bound follows from the monotonicity of $\gamma_2(\cdot)$, i.e., $\gamma_2(M_f)$  increases as $T$ increases (\Cref{claim:monotonicity}). 

For the upper bound, let $\mathcal M_f$ be the Toeplitz operator whose principal submatrix is $M_f$, i.e., $\mathcal M_f[i,j]=f(i-j+1)$. Consider a lower-triangular Toeplitz operator, $\mathcal H$ with $\mathcal H[i,j]=a_{i-j}$ (for $i \geq j$) for a sequence  $(a_n)_{n\geq 0}$ computed below such that $\mathcal H^2 = \mathcal M_f$. From matrix multiplication and comparing the coefficients, we have that the associated symbols, 
$g(x) = 1 + f(2)x + f(3) x^2 + f(4)x^3 + \cdots$ and $h(x) = a_0 + a_{1} x + a_{2} x^2 + a_{3} x^3 + \cdots $, satisfy $h(x)^2=g(x)$ (also see \Cref{lem:neighborhood}). In other words, we are interested in the exact closed-form expression for the coefficients of $h(x)$. 

To compute it, note that the $n$-th derivative of $h(x)$ is 
$h^{(n)}(x)=n! a_{n} + (n+1)_n a_{n+1} x + \cdots$, which equals $n! a_{n}$ for $x = 0$. Here, $(m)_n = m(m-1)(m-2) \cdots (m-n+1)$ denotes the falling factorial. We can also compute $h^{(n)}(x)$ by using the Faa di Bruno formula as follows. 
We set $G(x)=g(x)$ and $F(y)=\sqrt{y}$ in  Faà di Bruno's formula (\Cref{thm:faadibruno}) to get
\begin{align*}
    {\mathsf d^n \over \mathsf dx^n} F(G(x)) &= \sum_{k=1}^n \left({1 \over \paren{G(x)}^{k-1/2}} B_{n,k}\left(G^{(1)}(x),G^{(2)}(x),\dots\right) \prod_{m=0}^{k-1}\paren{{1 \over 2}-m} \right).
\end{align*}
Combining the two sides shows that
\[
a_{n} = {1 \over n!} \lim_{x \to 0}  \sum_{k=1}^n {1 \over \paren{G(x)}^{k-1/2}} B_{n,k}\left(G^{(1)}(x),G^{(2)}(x),\dots \right) \prod_{m=0}^{k-1}\paren{{1 \over 2}-m}.
\]
Let $(m)_\ell$ denote the falling factorial. The value of $a_n$ now follows by noting that  $$G^{(k)}(x) = k! f(k+1) + \sum_{j\geq 1}(k+j)_k f(k+j+1) x^j.$$
\Cref{thm:gamma2norm} now follows by setting $L=R$ to be the $T \times T$ principal submatrix of $\mathcal H$.

\begin{remark}
 [Computing the coefficients]
 \label{rem:coefficient}
Given a polynomial $g(x)$ with formal power series, we can compute the first $T$ coefficients of $\sqrt{g(x)}$ in $T^2$ time by identifying the coefficients by increasing powers. That is, comparing the coefficients of $(1 + a_1x + \cdots)^2$ and $g(x)$ to compute $a_1, a_2, \cdots, a_T$ in order. \end{remark}

\subsection{Proof of \Cref{cor:gamma2norm}}
Before proving \Cref{cor:gamma2norm}, we prove a key result. Note that, if $F(x)=\sqrt{x}$ and $a=1$ in \Cref{thm:inversebellpolynomial}, then $F(1+x)$ is the special case of the {\em potential polynomials}~\cite{comtet1974advanced} and has a formal power series: if ${1/2 \choose n}$ are generalized Binomial coefficients with $n \in \mathbb N$, then 
\[
F(1+x) = \sqrt{1+x} = {1/2 \choose 0} + {1/2 \choose 1} x + {1/2 \choose 2} x^2 + \cdots 
\]
 Using the inversion formula (\Cref{thm:inversebellpolynomial}), we show the following result in \Cref{sec:squarerootcoefficients}. 

\begin{theorem}
\label{cor:squarerootcoefficients}
    Let $F(x) = \sqrt{x}$. 
    Then for any sequence $s_1, s_2, \cdots$, the following inverse relationship holds: 
    \begin{align*}
        s_n = 2y_n + \sum_{\ell=1}^{n-1} {n \choose \ell}y_\ell y_{n-\ell} \quad  \text{and} \quad 
        y_n = \sum_{k=1}^n F^{(k)}(x)|_{x=1} B_{n,k}(s_1,s_2, \cdots),
    \end{align*}
    where $F^{(k)}(x)|_{x=1}$ denotes the evaluation of $k$-th derivative of $F(x)$ at $x=1$.
\end{theorem}

We complete the proof using \Cref{cor:squarerootcoefficients}. 
Due to \Cref{thm:gamma2norm}, we are left with bounding $a_n$ in \cref{eq:sn}.
Define the function $g(x) = 1 + \sum_{i\geq 1} f(i+1) x^i$.
Since $f \in \mathcal F$,  
$\sqrt{g(x)} = 1 + a_{1} x + a_{2} x^2 + a_{3} x^3 + \cdots $ is the formal power series for $\sqrt{g(x)}$ with
$a_n \geq 0$ for all $n \geq 0.$ 
Let $F(x) = \sqrt{x}$. Then  
\[
F^{(k)}(x) = x^{1/2-k} \prod_{m=0}^{k-1} \paren{{1\over 2} -m}.
\]

By setting  $s_n = n ! f(n+1)$ for all $n \geq 1$  in  \Cref{cor:squarerootcoefficients}, we  have    
\[
y_n = \sum_{k=1}^n  B_{n,k}\left(1! f(2), 2! f(3) ,\dots,\right) \prod_{m=0}^{k-1}\paren{{1 \over 2}-m} \quad \text{and} \quad s_n = 2y_n + \sum_{\ell=1}^{n-1} {n \choose \ell}y_\ell y_{n-\ell}.
\]

We now identify that $y_n=n! a_n$.
As discussed above,
all the coefficients of formal series are non-negative, i.e., $a_n \geq 0$ for all $n \geq 0$. This implies that $y_n \geq 0$ for all $n\geq 0$. Therefore, 
\[
s_n = 2y_n + \sum_{\ell=1}^{n-1} {n \choose \ell}y_\ell y_{n-\ell} \geq 2y_n = 2  a_n n!
\]

\noindent Since $s_n ={n!} f(n+1)$, we have 
that $2a_n \leq f(n+1)$, which gives \cref{eq:gamma2norm} in 
\Cref{cor:gamma2norm}. 
The statement for $f(n) = n^{-c}$ follows by
simple mathematical manipulations and is given in \Cref{sec:generalc}.

\subsection{Proof of \Cref{thm:exponential}}
\label{sec:exponential}
We start with a simple claim and then proceed to the proof of \Cref{thm:exponential}. 
\begin{claim}\label{claim:s}
 Let   $S_{T,2\alpha}=\sum_{n=1}^{T-1} {1 \over n\alpha^{2n}}$. For $\alpha > 1$, $S_{T, 2\alpha} \le 
 {\alpha^2 \over (\alpha^2-1)^2} - {\alpha^2 \over T(\alpha^2-1)\alpha^{2T}} $.
\end{claim}
\begin{proof}
 Since $\alpha>1$, subtracting $S_{T,2\alpha}/\alpha^2$ from $S_{T,2\alpha}$  gives us 
 \begin{align*}
    S_{T,2\alpha}\paren{1 -{1 \over \alpha^2}} &= {1 \over \alpha^2} \paren{1 + {1 \over \alpha^2} + {1 \over \alpha^4} + \cdots + {1 \over \alpha^{2(T-2)}} } -  {\alpha^2 \over T\alpha^{2T}} 
    \leq {1 \over \alpha^2-1} - {1 \over T\alpha^{2T}}.
\end{align*}
Dividing by $1-1/\alpha^2$ gives the desired claim..
\end{proof}

Equipped with \Cref{claim:s}, we complete the proof of \Cref{thm:exponential}.
When $f(n)=\alpha^{-n}$, the formal power series is $g(x)= 1 + {x \over \alpha} + {x \over \alpha}^2 + \cdots $
and 
\[
\sqrt{g(x)} = \paren{1 - {x \over \alpha}}^{-1/2} = 1 - {-1/2 \choose 1}{x\over \alpha} + {-1/2 \choose 2}{x^2 \over \alpha^2} + \cdots 
\]
Therefore, using Chen and Qi~\cite{chen2005best} and \Cref{thm:gamma2norm},  
\begin{align}
a_n = {1 \over \alpha^n} \left \vert {-1/2 \choose n}\right\vert \leq {1 \over \alpha^n} {1 \over \sqrt{\pi n}} \quad \text{which implies} \quad 
\gamma_2(M_f) \leq 1 + {1 \over \pi} \sum_{n=1}^{T-1} {1 \over n\alpha^{2n}}.
\label{eq:exponential_an}    
\end{align}
When $\alpha=1$, the summation in (\ref{eq:exponential_an}) is the Harmonic sum, yielding the bound in  \cite{henzinger2022constant, henzinger2022almost}. When $\alpha>1$, plugging the bound of $S_{T,2\alpha}$ from Claim~\ref{claim:s} into \cref{eq:exponential_an} gives the theorem statement.

\subsection{Proof of \Cref{thm:privatecounting}: General Framework for CDS Problem}\label{sec:framework}

Let $f:\mathbb N_+ \to \mathbb R_{+}$ be a monotonically non-increasing function, and for a stream, $x_1, x_2, \cdots, x_T$ of real numbers, we are required to output at every time $1 \leq t \leq T$ the following value: 
$\sum_{i=1}^t x_i f(t-i+1).$

This can be represented as matrix-vector multiplication, and  at time $t$, the output is the $t$-th coordinate of the matrix-vector product $M_f x$, where $x$ is the column vector formed by the streamed input.
Now consider a Toeplitz operator $\mathcal M_f$, whose $T \times T$ principal submatrix is $M_f$. It is easy to see that one can write the $(i,j)$ entry of $\mathcal M_f$, $\mathcal M_f[i,j] = f(i-j+1)$. Let us associate a polynomial that succinctly represents the operator $\mathcal M_f$, also called the \emph{ associated symbol}:
\[
g(x) = \sum_{k\geq 1} f(k)x^{k-1}
\]
Now consider the  polynomial $\mathsf r (x)=\sum_k r(k)x^{k-1}$ such that $(\mathsf r (x))^2 = g(x) $ and $r:\mathbb N \to \mathbb R_{+}$ (see \Cref{rem:coefficient}). 
From B\"{o}ttcher and  Grudsky~\cite{bottcher2000toeplitz}, it follows that the function $\mathsf r (x)$ is the  associated symbol of the Toeplitz operator $\mathcal L_f$ that satisfies $\mathcal L^2_f = \mathcal M_f$. 
Given a function $f$, if we find such a Toeplitz operator, it must be lower-triangular as its associated symbol cannot have any negative exponents. 
We can thus use the factorization mechanism with an appropriately scaled Gaussian vector to solve the CDS problem. This is described in \Cref{alg:factorizationmechanism}. We can use this algorithm to prove \Cref{thm:privatecounting}. For any factorization $L$ and $R$ of a matrix $A$
it has been shown by Edmonds et al.~\cite{edmonds2020power} that
$
\absoluteerror(\mathsf M_{L,R}, A,T) \leq  \variance \Delta \gamma_2(A) \sqrt{\log (T)}
$ 
and by Li et al.~\cite{li2010optimizing} that
$
\meansquared(\mathsf M_{L,R}, A,T) \le {\frac{1}T }\variance^2 \Delta^2 \gamma_{\op{F}}(A)^2.$   
As $F = L L$, the result follows because $\gamma_2(M_f) = \sum_{i=1}^T r(i)^2$ and
$\gamma_F(M_f) \le \sqrt{T} \gamma_2(M_f)$.  The privacy guarantee follows from the privacy of the factorization mechanism~\cite{li2015matrix}. 

\subsection{Proof of \Cref{cor:privatecounting}}
We recall the result of Fichtenberger et al.~\cite{henzinger2022constant}, restated in our notation\footnote{\cite{henzinger2022constant} states their result in terms of $\Psi(T)$, which is the same as $\norm{L}_{1 \to 2} \norm{R}_{2 \to \infty}$ if $LR$ is the factorization of the counting matrix ($f(n)=1$)}:
\begin{theorem}
[Theorem 2 in \cite{henzinger2022constant}]
Let $\epsilon,\delta \in (0,1)$ be the privacy parameters. There is an efficient $(\epsilon,\delta)$-differentially private CDS algorithm $\mathcal A(M_f)$  for the constant function $f(n)=1$. The algorithm $\mathcal A$ outputs $s_t$ in round $t$ such that
in every execution, with a probability at least $2/3$ over the coin tosses of the algorithm, 
simultaneously, for all rounds $t$ with $1 \le t \le T$, it holds that
\begin{align}
   \left\vert s_t - \sum_{i=1}^t x_i \right\vert \leq   \variance \norm{L}_{1 \to 2} \norm{R}_{2 \to \infty} \sqrt{\log(T)}, 
   \label{eq:binary_counting} 
\end{align} 
where $M_f = LR$. 
\end{theorem}

We now return to the proof of \Cref{cor:privatecounting}
Note that our factorization is a lower-triangular matrix and the $\ell_2$-sensitivity  is
$\Delta \sqrt{\sum_{i=1}^T r(i)^2}$. Since we chose
$\sigma^2 = \variance \Delta^2 \sum_{i=1}^T r(i)^2$. Thus 
 $(\epsilon,\delta)$-differentially privacy follows from \Cref{thm:gaussian} as in \cite{choquette2022multi, mcmahan2022private, edmonds2020power, henzinger2022constant, henzinger2022almost}.

To show the bound on the additive error we first note that the proof of Theorem 2 in~\cite{henzinger2022constant} only depends on $L$ and $R$, i.e., all dependence on $f$ is limited to the dependence on $L$ and $R$. Thus, we can use their analysis for any monotonically non-increasing function $f:\mathbb N_+ \rightarrow \R_+$ and our corresponding bounds on the factorization of the corresponding $M_f$. In particular, the algorithm $\mathcal A(M_f)$ takes as input a stream of values $x_1, \cdots, x_T$ and output $s_t$ in round $t$ such that, in every execution, simultaneously for all rounds $t$ with $1 \le t \le T$, we have the following bound depending on the function $f$. 
    \begin{enumerate}
        \item $f(n)=(n+1)^{-c}$ for some constant c. 
        \[
        \left\vert s_t - \sum_{i=1}^t {x_i \over (t-n+1)^c } \right\vert \leq  \variance \polynomialbound \sqrt{\log(T)}
        \]
        because $\norm{L}_{1 \to 2} \norm{R}_{2 \to \infty} \leq \polynomialbound$ (\Cref{thm:generalc}).

        \item $f(n)=\alpha^{-1}$ for some $\alpha >1$. 
        \[
        \left\vert s_t - \sum_{i=1}^t {x_i \over (t-n+1)^c } \right\vert \leq  \variance \exponentialbound \sqrt{\log(T)}
        \]
        because $\norm{L}_{1 \to 2} \norm{R}_{2 \to \infty} \leq \exponentialbound$ (\Cref{thm:exponential}). 
    \end{enumerate}
    
    Further, for the $\ell_2^2$-error, using \Cref{thm:generalc},  \Cref{thm:exponential}, and \cref{eq:meansquaredgammanorm}, the following holds: 
    \begin{enumerate}
        \item When $f(n)=(n+1)^{-c}$ for some constant c, then 
        \[
        \meansquared(M_f,T) \leq \variance^2 \polynomialbound^2
        \]
        \item When $f(n)=\alpha^{-1}$ for some $\alpha >1$, then 
        \[
        \meansquared(M_f,T) \leq \variance^2 \exponentialbound^2
        \]
    \end{enumerate}
    Finally, the case when $f \in \mathcal F$ follows using \Cref{cor:gamma2norm} along with \cref{eq:absoluteerrorgammanorm} and \cref{eq:meansquaredgammanorm}.

\subsection{Proof of \Cref{thm:slidingwindowmodel}: Extension to the Sliding Window Model}
\label{sec:slidingwindow}
{
The sliding window model is parameterized by a parameter $w \in \mathbb N_+$
and the function $f$ is 
\[
f(n) = \begin{cases}
    1 & n <w \\
    0 & \text{otherwise}
\end{cases}.
\]
Note that the matrix still has the Toeplitz structure. 
We now consider the following block matrix and use its output to compute the sliding window sum:
\[
M_f' = \begin{pmatrix}
    M_1 & 0^{w \times w} & 0^{w \times w} & \cdots & 0^{w \times w} \\
    0^{w \times w} &M_1 & 0^{w \times w} & \cdots & 0^{w \times w} \\
    \vdots & \vdots & \ddots & \vdots & \vdots \\
    0^{w \times w} & 0^{w \times w} &  0^{w \times w} & M_1& 0^{w \times w} \\ 
    0^{w \times w} & 0^{w \times w} &  0^{w \times w}& 0^{w \times w} & M_1
\end{pmatrix},
\]
where $M_1 \in \set{0,1}^{w \times w}$ is a lower-triangular matrix ({including the diagonal}) with all non-zero entries  being one. In the above, we assume that $T$ is a multiple of $w$. If not, then the last block matrix would be the first $a \times a$ principal submatrix of $M_1$ for $a \equiv T (\mod w)$ instead of $M_1$.

Now to compute a sum over the sliding window model at time $t$, we use the following procedure:
\begin{enumerate}
    \item Define a vector $x' \in \real^T$ as follows:
    \[
    x'[i] = \begin{cases}
        x[i] & i \leq t \\
        0 & i >t
    \end{cases}.
    \]
    \item If $ t \leq w$, the output the $t$-th coordinate of $M_f' x'$.
    \item If $t >w$, then output the difference between $t$-th and $(t-w)$-th coordinate of $M_f'x'$: 
    \[
    (M_f' x')[t] - (M_f' x')[t-w].
    \]
\end{enumerate}

Now we can compute the factorization of $M_1$ as in the case of $f(n)=1$. Let $L_1$ be the factorization such that $L_1^2 =M_1$. Then, we use the following block matrix as a factorization:
\[
L_f' = 
 \begin{pmatrix}
    L_1 & 0^{w \times w} & 0^{w \times w} & \cdots & 0^{w \times w} \\
    0^{w \times w} & L_1 & 0^{w \times w} & \cdots & 0^{w \times w} \\
    \vdots & \vdots & \ddots & \vdots & \vdots \\
    0^{w \times w} & 0^{w \times w} &  0^{w \times w} & L_1& 0^{w \times w} \\ 
    0^{w \times w} & 0^{w \times w} &  0^{w \times w}& 0^{w \times w} & L_1
\end{pmatrix}.
\]

The $(\epsilon,\delta)$-differentially private algorithm is therefore {the following online variant of the factorization mechanism}:
\begin{enumerate}
    \item Compute a factorization $L_f'$ and the vector $x'$ (from the stream) as above.
    \item Sample a random vector $z \sim \mathcal N(0, \sigma^2 LL^\top)$, where $\sigma^2 = \variance^2 \Delta^2 \norm{L_1}_{1 \to 2}^2$. 
    \item At time $t$, compute $v = L_f'L_f'x$. Store $u[t]:=v[t] + z[t]$ 
    \item If 
         $t\leq w$, then output $u[t]$; otherwise, output $u[t] - u[t-w]$. 
\end{enumerate}

We first show that the algorithm is $(\epsilon,\delta)$-differentially private.
Note that the algorithm first applies the Gaussian mechanism to $L_f'$ and then post-processes the result by multiplying it with $L_f'$. As post-processing does not affect the privacy guarantees, it suffices to show that we chose the variance $\sigma^2$ in the Gaussian mechanism suitably to guarantee $(\epsilon,\delta)$-differentially privacy.

Let $g: [\Delta, \Delta]^T \rightarrow \mathbb{R}^T$ be the function $L_f' x$ with
$x \in [\Delta, \Delta]^T$.
Note that its $\ell_2$-sensitivity  is
$$\Delta \sqrt{\sum_{i=1}^T r(i)^2} = \Delta \norm{L_f'}_{1 \to 2} = \Delta \norm{L_1}_{1 \to 2},$$ where $r$ is the associated symbol of $L_f'$, which equals the associated symbol of $L_1$.
Now $L_1$ equals the factorization used for continual binary counting and, thus, 
the fact that $$\sum_{i=1}^T r(i)^2 \le 1 + {\log(w) \over \pi} + {2\over w} $$
follows from
~\cite{henzinger2022constant}.
By our choice of $\sigma^2$ 
 $(\epsilon,\delta)$-differentially privacy follows from \Cref{thm:gaussian}.

\section{Detailed Comparison with Prior Work}
\subsection{Comparison of $\gamma_2$ and $\gamma_F$ norm}
The prior bound on $\gamma_2$ is 
\begin{align}
 \gamma_2(M_f) \leq \sqrt{H_{T,c}} \quad \text{for $f(n)=n^{-c}$} \quad \text{and} \quad  \gamma_2(M_f) \leq \sum_{n=0}^{T-1} {1 \over \alpha^{2n}} \quad \text{when $f(n)=\alpha^{-n}$}
\label{eq:naivepolynomialgammanorm}
\end{align}
The prior best known bound for $\gamma_F$ (for both $f(n)=\alpha^{-n}$ and $f(n)=n^{-c}$) follows from the inequality $\gamma_F(M_f)^2 \leq T \gamma_2(M_f)^2$. 
We now show that our bounds improve on these prior bounds: 
\begin{lem}
\label{lem:comparisongammanorm}
    The $\gamma_2$ and $\gamma_F^2$ norm computed in \Cref{cor:gamma2norm} and \Cref{thm:exponential} are better than  the best-known bound mentioned above in \cref{eq:naivepolynomialgammanorm}.
\end{lem}
\begin{proof}
    To show the lemma for $f(n)=n^{-c}$ for $c\geq 1$ and the $\gamma_F^2$ norm, it suffices to show that $p(H_{T,2c}) <0$, where  $p(y)=\paren{1 + {y-1 \over 4}}^2 -y$. 
    Now, $p(y) < 0$ for all $y \in (1,9)$. We know that (a) $H_{T,2c} > 1$ for all $c \in \mathbb N_+$ and $T > 1$ and (b) $H_{T,2c} \leq \zeta(2c)$, where $\zeta(2c)$ is the Reimann zeta function of order $2c$. Further, for all even orders, the Riemann zeta function is a monotonically decreasing function and is lower bounded by $1$. That is, $1 \leq \zeta(2c) \leq \zeta(2) = \pi^2/6 <1.65$. 
    Thus, $H_{T,2c} \in (1,9)$ and, hence,
    $p(H_{T,2c})$ is always negative.
    The above argument allows us to prove the claim for $\gamma_2$-norm as follows. We have shown that $p(y) = \paren{\paren{1 + {y-1 \over 4}} - \sqrt{y}} \paren{\paren{1 + {y-1 \over 4}} + \sqrt{y}}<0$ for $y=H_{T,2c}$. Since ${\paren{1 + {y-1 \over 4}} + \sqrt{y}}>0$  for $y=H_{T,2c}$, we have ${\paren{1 + {y-1 \over 4}} - \sqrt{y}}<0$ for $y=H_{T,2c}$ as required in the claim for $\gamma_2$ norm for $f(n)=n^{-c}$.

    To show the claim for squared $\gamma_F$-norm when  $f(n)=\alpha^{-n}$, it suffices to show that 
\begin{align}
    \paren{1 + \frac{1}{\pi} \sum_{n=1}^{T-1} \frac{1}{n \alpha^{2n}}}^2 < \sum_{n=0}^{T-1} \frac{1}{\alpha^{2n}}.
    \label{bound}
\end{align}

Note that the left side is a polynomial in ${1\over \alpha^2}$ such that the coefficient of $\alpha^{-2n}$ for $n>0$ is ${1 \over \pi} + \sum_{k=1}^{n-1}\frac{1}{\pi^2 k (n-k)}$. The coefficient for $n=0$ is on both sides 1. 
We will now proceed in two steps:
(1) We will show that $\sum_{k=1}^{n-1}\frac{1}{\pi^2 k (n-k)} \le \frac{2\log n}{\pi^2 n}$.
(2) We use that to show Inequality (\ref{bound}).

Note that $\frac{1}{k (n-k)} = \frac{1}{nk} + \frac{1}{n(n-k)}$.
Thus, $\sum_{k=1}^{n-1}\frac{1}{\pi^2 k (n-k)}
= \frac{1}{\pi^2 n} \left(\sum_{k=1}^{n-1} \paren{ \frac{1}{k} + \frac{1}{n-k} }\right) = \frac{2}{\pi^2 n} \sum_{k=1}^{n-1} \frac{1}{k} \le \frac{2 \log n}{\pi^2 n}.$ Now the following set of inequalities gives Inequality (\ref{bound}):
\begin{align*}
\left(1 + \frac{1}{\pi} \sum_{n=1}^{T-1} \frac{1}{n \alpha^{2n}} \right)^2 
    &\le 
    1 + \sum_{n=1}^{2T-2} \left(\frac{1}{\pi} +  \frac{2 \log n}{\pi^2 n} \right) \frac{1} {\alpha^{2n}} \\
    &=1 + \sum_{n=1}^{T-1}  \left(\frac{1}{\pi} +  \frac{2 \log n}{\pi^2 n} \right) \frac{1} {\alpha^{2n}}+ \sum_{n=T}^{2T-2}  \left(\frac{1}{\pi} +  \frac{2 \log n}{\pi^2 n} \right) \frac{1} {\alpha^{2n}}\\
      &\le 1 + \frac{2}{\pi} \sum_{n=1}^{T-1}\frac{1} {\alpha^{2n}}  + \frac{4}{\pi^2} \sum_{n=1}^{T-1} \frac{\log n}{ n \cdot \alpha^{2n}} 
    1 + \sum_{n=1}^{T-1}\left(\frac{2}{\pi} +  \frac{4\log n}{\pi^2 n}\right) \frac{1} {\alpha^{2n}}\\
     &< 1 + \sum_{n=1}^{T-1} \frac{1}{\alpha^{2n}}
     = \sum_{n=0}^{T-1} \frac{1}{\alpha^{2n}}.
\end{align*}

For $\gamma_2(M_f)$ when $f(n)=\alpha^{-n}$, note that 
inequality~(\ref{bound}) implies the following set of inequalities.
\begin{align*}
    0 & < \paren{1 + \frac{1}{\pi} \sum_{n=1}^{T-1} \frac{1}{n \alpha^{2n}}}^2 - \sum_{n=0}^{T-1} \frac{1}{\alpha^{2n}} \\
    &= \paren{{1 + \frac{1}{\pi} \sum_{n=1}^{T-1} \frac{1}{n \alpha^{2n}}} - \paren{\sum_{n=0}^{T-1} \frac{1}{\alpha^{2n}}}^{1/2}} \paren{{1 + \frac{1}{\pi} \sum_{n=1}^{T-1} \frac{1}{n \alpha^{2n}}} + \paren{\sum_{n=0}^{T-1} \frac{1}{\alpha^{2n}}}^{1/2}}.
\end{align*}
Since the second term in the product is positive, we have the claim. 
\end{proof}

\subsection{Comparison with the Gaussian mechanism}
\label{sec:gaussiancomparison}
\label{lem:comparison}
Recall that the Gaussian mechanism simply adds at each time step $1 \leq t \leq T$ noise whose standard deviation is proportional to the $\ell_2$-sensitivity of the prefix sums. 
The $\ell_2$-sensitivity of the CDS problem is $L_{\text{CDS}} = \sqrt{\sum_{n=1}^T f(n)^2}$. Since the additive $\ell_\infty$-error
of Gaussian mechanism is $\variance \Delta L_{\text{CDS}} \sqrt{\log T}$ 
and the
$\ell_2^2$-error is  $\variance^2 \Delta^2 L_{\text{CDS}}^2$, the claim follows as  a corollary of \Cref{lem:comparisongammanorm}.

\subsection{Comparison with Bolot et al.~\cite{bolot2013private}}
\label{sec:bolot}
To understand the nature of the pure additive error, we consider  the case when the stream is all zero, the best case for \cite{bolot2013private} (our algorithm is oblivious to the stream). As the additive error of all the algorithms scales with sensitivity, we only compare the bound on sensitivity  for different algorithms. Bolot et al.~\cite[eq. (8)]{bolot2013private} compute the overall sensitivity of their mechanism for polynomial decay function to be $L_{\text{BFMNT}} = {1\over \beta} + {1\over c\beta^2} \log \paren{1 \over 1-\beta} > 1 + {1 \over c}$. Our algorithm achieves better additive error since the sensitivity of our algorithm is $\gamma_2(M_f) \leq  1 + {\zeta(2c)-1 \over 4} \leq 1 + \frac{1}{4(2c-1)}$ using \Cref{cor:gamma2norm}.

\subsection*{Acknowledgements.} 
This project has received funding from the European Research Council (ERC) 
\begin{wrapfigure}{r}{0.15\textwidth}\label{fig:diff}
\includegraphics[width=0.13\textwidth]{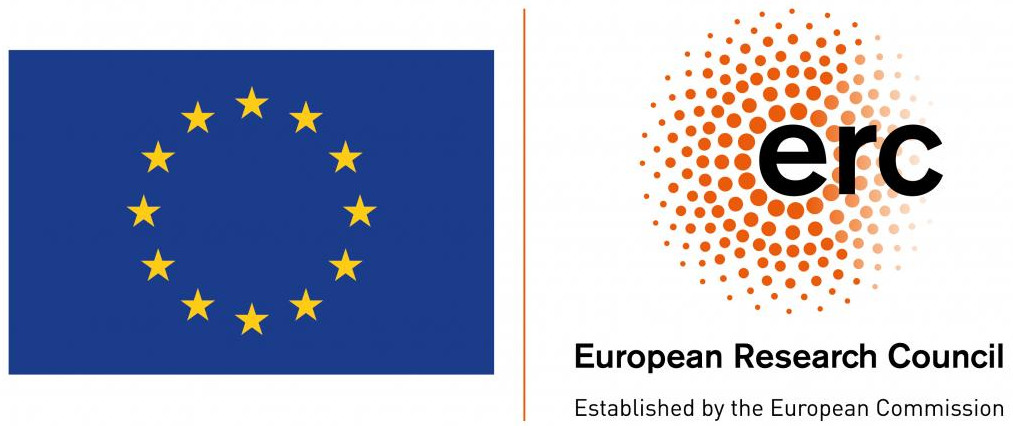}
\end{wrapfigure}
under the European Union's Horizon 2020 research and innovation programme
(Grant agreement No.\ 101019564 ``The Design of Modern Fully Dynamic Data
Structures (MoDynStruct)'' and from the Austrian Science Fund (FWF) project
Z 422-N, and project 
``Fast Algorithms for a Reactive Network Layer (ReactNet)'', P~33775-N, with
additional funding from the \textit{netidee SCIENCE Stiftung}, 2020--2024. JU's research was funded by Decanal Research Grant. A part of this work was done when JU was visiting Indian Statistical Institute, Delhi.

\bibliography{references}
\bibliographystyle{alpha}

\appendix
\section{Notations and Preliminaries for the Proofs in \Cref{sec:specialmatrices}}

\subsection{Polynomials.} 
The main family of polynomials we will be working with are related to Jonquière's function, which is parameterized by an integer $c$, known as the order of Jonquière's function. Formally, Jonquière's function (or polylogarithmic function) is defined as follows:
\[
\mathsf{Li}_c(x) = x + {x^2 \over 2^c} + {x^3 \over 3^c} + \cdots
\]
We only consider a restricted set of Jonquière's functions that have positive integer order.
We will be mainly interested in the function $${\mathsf{Li}_c(x) \over x} = 1 + {x \over 2^c} + {x^2 \over 3^c} + \cdots$$
when $c \in \mathbb N_+$. 

We note that $\mathsf{Li}_c(x)$ can be derived from $\mathsf{Li}_{c+1}(x)$ by the following recurrence relation:
\[
\mathsf{Li}_c(x) = x {\mathsf d \over \mathsf dx} \mathsf{Li}_{c+1}(x).
\]

Another function that we will deal with is the \emph{Gaussian} or \emph{ordinary hypergeometric function}, denoted by ${_2F_1}(a,b;c,z)$ and defined as 
\[
{}_2F_1(a,b;c;z) = \sum_{n=0}^\infty \frac{(a)_n (b)_n}{(c)_n} \frac{z^n}{n!} = 1 + \frac{ab}{c}\frac{z}{1!} + \frac{a(a+1)b(b+1)}{c(c+1)}\frac{z^2}{2!} + \cdots.
\]
with $|z| <1$
and 
\[
(a)_n=
\begin{cases}  1  & n = 0 \\
  a(a+1) \cdots (a+n-1) & n > 0
 \end{cases}
\]
is the falling factorial.

One useful property of the Gaussian function is that 
\[
\frac{d^n }{dz^n} \ {}_2F_1(a,b;c;z) = \frac{(a)_n (b)_n}{(c)_n} {}_2F_1(a+n,b+n;c+n;z)
\]

\subsection{Differential Privacy}
In all the use cases we cover in this paper, our privacy and utility guarantee depends on the Gaussian distribution. Given a random variable $X$, we denote by $X \sim N(\mu, \sigma^2)$ the fact that $X$ has Gaussian distribution with mean $\mu$ and variance $\sigma^2$ with the probability density function
\[
p_X(x) = \frac{1}{\sqrt{2 \pi \sigma}} e^{-\frac{(x-\mu)^2}{2\sigma^2}}.
\]

Our algorithm for continual counting uses the Gaussian mechanism. To define it, we need to first define the notion of $\ell_2$-sensitivity. For a function $f : \mathcal X^n \to \R^d$  its {\em $\ell_2$-sensitivity} is defined as 
\begin{align}
\Delta f := \max_{\text{neighboring }X,X' \in \calX^n} \norm{f(X) - f(X')}_2.
\label{eq:ell_2sensitivity}    
\end{align}

\begin{definition}
[Gaussian mechanism~\cite{dwork2014algorithmic}]
\label{def:gaussianmechanism}
Let $f : \mathcal X^n \to \R^d$ be a function with $\ell_2$-sensitivity $\Delta f$. For a given $\epsilon,\delta \in (0,1)$  given $X \in \mathcal X^n$
the Gaussian mechanism $\mathsf M$
returns $\mathsf M(X) =  f(X) + e$, where $e \sim N(0,\variance^2 (\Delta f)^2 \I_d)$. Here $\variance = {2 \sqrt{\log(1.25/\delta)} \over \epsilon}$. 
\end{definition}

\begin{theorem}
\label{thm:gaussian}
For a given $\epsilon,\delta \in (0,1)$ the Gaussian mechanism $\mathsf M$ satisfies $(\epsilon,\delta)$-differential privacy.
\end{theorem}

\section{Bounds on special matrices}
\label{sec:specialmatrices}
\subsection{Upper Bound on Factorization norm of $M_f$ when $f(n) = 1/n$}
The result in this section can be derived from \Cref{thm:generalc}, but we present a proof for the sake of completion. 
We will prove the following upper bound:
\begin{theorem}
\label{thm:c=1}
For the function $f(n)=n^{-1}$, it holds that 
\[
\gamma_2(M_f) \leq  \sum_{n=1}^T a_n^2, \quad  \text{where} \quad 
a_n =
\begin{cases}
{1\over n!} \sum_{k=1}^n B_{n,k} \paren{\frac{1}{2}, {2! \over 3}, \cdots, {(n-k+1)! \over (n-k+2)}} \prod_{m=0}^{k-1}\paren{{1 \over 2}-m} & n \geq 1 \\
1 & n =0
\end{cases}.
\]
One can further simplify the above bound to 
\[
\gamma_2(M_f) \leq 1 + \sum_{n=1}^T {1 \over 4(n+1)^2} \leq {9 \over 8} - {1 \over 4(T+1)}
\]
\end{theorem}

\begin{proof}
Let $\mathcal F$ be the Toeplitz operator for the function $f(i) = 1/i$. Note that its  associated
polynomial equals ${\mathsf{Li}_1 (x) \over x }$.
Following the algorithm in \Cref{sec:framework}
we first determine the 
 the Toeplitz operator $\mathcal L$ such that $\mathcal L^2 = \mathcal F$.
This is equivalent to determining 
the coordinates in the polynomial expansion of 
\[
\sqrt{{\mathsf{Li}_1 (x) \over x }} = 1 + a_1 x + a_2 x^2 + \cdots 
\]
We note that the $n$-th coefficient $a_n$ of $\sqrt{\frac{\mathsf{Li}_1(x)}x}$ is simply $1$ for $n=0$ and for $n>0$,
\[
a_n = \frac{1}{n!} \lim_{x \to 0} \frac{\mathsf d^n}{\mathsf dx^n}\sqrt{\frac{\mathsf{Li}_1(x)}x}
\]
So, we need to compute the above limit. 
As $\frac{\mathsf{Li}_1(x)}{x} =  {_2F}_1(1,1;2,-x) = {-\ln(1-x) \over x}$ and  we are interested in computing 
\[
\lim_{x \to 0}{\mathsf d^n \over \mathsf dx^n} \sqrt{{-\ln(1-x)\over x}} = \lim_{x \to 0}{\mathsf d^n \over \mathsf dx^n} \sqrt{_2F_1(1,1;2,-x)}
\]
As we know that 
\[
{\mathsf d^n \over \mathsf dx^n}{_2F_1}(1,1;2,-x) = {(1)_n(1)_n \over (2)_n} {_2F_1}(n+1,n+1;n+2,-x)
\]
Setting $g(x)={_2F_1(1,1;2,-x)}$ and $f(y)=\sqrt{y}$ in  Faà di Bruno's formula, we have 
\begin{align*}
    {\mathsf d^n \over \mathsf dx^n} f(g(x)) &= \sum_{k=1}^n {1 \over \paren{-\ln(1-x)\over x}^{k-1/2}} B_{n,k}\left(g'(x),g''(x),\dots,g^{(n-k+1)}(x)\right) \prod_{m=0}^{k-1}\paren{{1 \over 2}-m}.
\end{align*}
Now 
\[
\lim_{x \to 0} {\mathsf d^n \over \mathsf dx^n}{_2F_1}(1,1;2,-x) = {(n!)^2 \over (n+1)!} = {n! \over (n+1)}
\quad \text{and} \quad 
\lim_{x \to 0} \paren{{-\ln(1-x) \over x}}^{k-1/2} = 1
\]
Therefore, 
\[
a_n={1\over n!} \lim_{x \to 0} {\mathsf d^n \over \mathsf dx^n} \sqrt{{-\ln(1-x) \over x}} = {1\over n!} \sum_{k=1}^n B_{n,k} \paren{\frac{1}{2}, {2! \over 3}, \cdots, {(n-k+1)! \over (n-k+2)}} \prod_{m=0}^{k-1}\paren{{1 \over 2}-m}.
\]
This completes the first part of  \Cref{thm:c=1}. The second part of \Cref{thm:c=1} follows from \Cref{cor:gamma2norm} since $f\in \mathcal F$.
\end{proof}

\subsection{Factorization norm of $M_f$ when $f(n) = 1/n^c$ for some natural number $c \geq 2$}
\label{sec:generalc}
We recall the function,
\[
\mathsf{Li}_c(x) = x + {x^2 \over 2^c} + {x^3 \over 3^c} + \cdots 
\]

\noindent As in the proof of \Cref{thm:c=1} the square root of 
\[
{\mathsf{Li}_c(x) \over x} = 1 + {x \over 2^c} + {x^2 \over 3^c} + \cdots 
\]
is the polynomial we need to analyze. 

\begin{theorem}
\label{thm:generalc}
Let $M_f$ be the matrix whose entries are defined by the function $f(n)=n^{-c}$ for some constant $c \geq 2$. Then 
\[
\gamma_2(M_f) \leq 1+ \sum_{n=1}^T a_n^2, \quad  \text{where} \quad 
a_n =
\begin{cases}
{1\over n!} \sum_{k=1}^n B_{n,k} \paren{\frac{1}{2^c}, {2! \over 3^c}, \cdots, {(n-k+1)! \over (n-k+2)^c}} \prod_{m=0}^{k-1}\paren{{1 \over 2}-m} & n \geq 1 \\
1 & n =0
\end{cases}.
\]
Further, $\gamma_F(M_f) \leq \sqrt{T}\paren{1 + \sum_{n=1}^T a_n^2}$.

When we care about asymptotics, we can bound each of the $a_n \leq {1 \over 2(n+1)}$ for $n \geq 1$ yielding the following bound:
\begin{align*}
    \gamma_2(M_f) &\leq  \paren{1 + \sum_{n=1}^T {1 \over 4(n+1)^{2c}}} \leq \paren{1 + {1 \over 4(2c-1)} -  {(T+1)^{1-2c} \over 4(2c-1)} } \\
    \gamma_F(M_f) &\leq \sqrt{T} \left(1+ \sum_{n=1}^T \polynomialbound \right)
\end{align*}
\end{theorem}

\begin{proof}
Let us denote by 
\[
\sqrt{{\mathsf{Li}_c(x)\over x}} = 1 + a_1 x +  a_2 x^2 + \cdots 
\]
Setting $g(x)={\mathsf{Li}_c(x)\over x}$ and $f(y)=\sqrt{y}$ in  Faà di Bruno's formula, we have 
\begin{align*}
    {\mathsf d^n \over \mathsf dx^n} f(g(x)) &= \sum_{k=1}^n {1 \over \paren{{\mathsf{Li}_c(x)\over x}}^{k-1/2}} B_{n,k}\left(g'(x),g''(x),\dots,g^{(n-k+1)}(x)\right) \prod_{m=0}^{k-1}\paren{{1 \over 2}-m}.
\end{align*}
Similar to the $c=1$ case, we have 
\[
\lim_{x \to 0} g^{n}(x) = {n! \over (n+1)^c}
\quad \text{and} \quad 
\lim_{x \to 0} \paren{{\mathsf{Li}_c(x)\over x}}^{k-1/2} = 1
\]
Therefore, 
\begin{align}
a_n:={1\over n!} \lim_{x \to 0} {\mathsf d^n \over \mathsf dx^n} \sqrt{{\mathsf{Li}_c(x)\over x}} = {1\over n!} \sum_{k=1}^n B_{n,k} \paren{\frac{1}{2^c}, {2! \over 3^c}, \cdots, {(n-k+1)! \over (n-k+2)^c}} \prod_{m=0}^{k-1}\paren{{1 \over 2}-m}.
\label{eq:b_ngeneralcase}    
\end{align}
giving the first part of \Cref{thm:generalc}. 
The second part of \Cref{thm:generalc} follows from \Cref{cor:gamma2norm} since $f\in \mathcal F$. 
Using the fact that $\gamma_F(A) \leq \sqrt{T} \gamma_2(A)$ for any $A \in \real^{T \times T}$ completes the proof of \Cref{thm:generalc}.
\end{proof}

\section{Missing Proofs from \Cref{sec:proofmaintheorem}}
\label{sec:missinglower}

\subsection{Auxiliary Claims in the Proof of \Cref{thm:gamma2norm}}
\label{sec:auxiliary}
Our lower bound relies on the following dual of the $\gamma_2(\cdot)$ norm~\cite{haagerup1980decomposition}: 
\begin{theorem}
[Haagerup~\cite{haagerup1980decomposition}]
\label{thm:haagerup}
For any complex linear operator $A$, 
\begin{align}
\gamma_2(A) = \max {\norm{A \bullet Q} \over \norm{Q}},
\label{eq:dual}    
\end{align}
where $\norm{Q}$ denotes the spectral norm and $A \bullet Q$ denotes the Schur product (or Hadamard product).
\end{theorem}

The trivial lower bound of $\gamma_2(M_f) \geq 1$ follows by  using the dual characterization of $\gamma_2(\cdot)$ norm by setting $Q$ to be identity. More precisely,
\[
\gamma_2(M_f) \geq \norm{\mathbb I_T} = 1
\]
since $A \bullet I$ results only in the diagonal entries of $A$, which in the case of $M_f$ is all one. 

Our lower bound improves this trivial lower bound. We use the following claim:
\begin{claim}
    \label{claim:norm}
    For $|a|<1$, let $\sin(\theta)= {\sqrt{4-2a^2 \over 4- a^2}}$,  
    \[
    Q = \begin{pmatrix}
        \sin(\theta) & \cos(\theta) \\
        -\cos(\theta) & \sin(\theta)
    \end{pmatrix}, \quad \text{and} \quad 
    A = \begin{pmatrix}
        1 & 0 \\
        a & 1
    \end{pmatrix}
    \]
    Then $\norm{A \bullet Q} = {2 \over \sqrt{4-a^2}}.$
\end{claim}
\begin{proof}
    First note that 
    \[
    B = A \bullet Q = \begin{pmatrix}
        \sin(\theta) & 0 \\
        -{a \cos(\theta)} & \sin(\theta)
    \end{pmatrix}, \quad \text{and} \quad
B^\top B = \begin{pmatrix}
    \sin^2(\theta) + a^2 \cos^2(\theta) & -a \cos(\theta) \sin(\theta) \\
    - a \cos(\theta) \sin(\theta) & \sin^2(\theta)
\end{pmatrix}
\]
Let the singular values of $B^\top B$ be $\lambda_1, \lambda_2$. 
Then $\lambda_1 + \lambda_2 = \tr{B^\top B} = 2\sin^2(\theta) + a^2 \cos^2(\theta)$ and $\lambda_1 \lambda_2 = \mathsf{det}(B^\top B) = \sin^4(\theta)$. Solving the system of equations, we have the claim.
\end{proof}

While it is folklore, to the best of our knowledge, the following result is not shown rigorously. 
\begin{claim}
[Monotonicity of $\gamma_2(\cdot)$]
\label{claim:monotonicity}
    Let $\mathcal A$ be a linear operator. For any $T \in \mathbb N_+$, let $A$ be the matrix formed by the $T \times T$ principal submatrix of $\mathcal A$ and let $\widehat A$ be the $(T+1) \times (T+1)$ principal submatrix of $\mathcal A$. Then $\gamma_2(A) \leq \gamma_2(\widehat A)$. 
\end{claim}
\begin{proof}
    Let $A \in \real^{T \times T}$ be the matrix formed by the first $T \times T$ principal submatrix of $\mathcal A$ and $\widehat A$ be formed by the first $(T+1) \times (T+1)$ principal submatrix of $\mathcal A$. 
    Let $Q \in \real^{T \times T}$ be the matrix that certifies the dual form of  $\gamma_2(A)$ (\cref{eq:dual}). 
    That is,
    $\gamma_2(A) = \norm{Q \bullet A}  \geq \norm{X \bullet A}$
    for all matrices $X$ such that $\norm{X}=1$. Let $v$ be the eigenvector that certifies $\norm{Q \bullet A}$, i.e.,
    $
    \gamma_2(A)^2 = \norm{(A \bullet Q)v}_2^2,
    $
    where $\norm{\cdot}_2$ is the $\ell_2$ norm. For ease of presentation, let us denote by $a = \widehat A[T+1,T+1]$
    We define the following unitary matrix: 
    \[
    \widehat Q = \begin{pmatrix}
        Q & 0 \\
        0^{1 \times T} & 1 
    \end{pmatrix} \in \real^{(T+1) \times (T+1)} \quad \text{which implies that} \quad \widehat A \bullet \widehat Q = \begin{pmatrix}
        A \bullet Q & 0 \\
        0^{1 \times T} & a 
    \end{pmatrix}.
    \]

    Note that $\norm{\widehat Q}=1$. From the dual characterization of $\gamma_2(\cdot)$, we have 
    \begin{align*}
    (\gamma_2(\widehat A))^2 & = \max_{\norm{X} =1} \norm{\widehat A \bullet X}^2 \geq \norm{\widehat A \bullet \widehat Q}^2  = \norm{
        \begin{pmatrix}
        A \bullet Q & 0 \\
        0^{1 \times T} &  a
        \end{pmatrix}}^2 \\&
        = \max_{\norm{x}_2=1} \norm{\begin{pmatrix}
        A \bullet Q & 0 \\
        0^{1 \times T} &  a  
        \end{pmatrix} x}^2 \geq \norm{\begin{pmatrix}
        A \bullet Q & 0 \\
        0^{1 \times T} &  a 
        \end{pmatrix} \begin{pmatrix}
            v \\ 0 
        \end{pmatrix}}_2^2  ={\norm{(A \bullet Q)v}^2 } =  (\gamma_2(A))^2,
    \end{align*}
    Since $\gamma_2(\cdot)$ is a norm, taking the square root on both sides completes the proof. 
\end{proof}

\subsection{Proof of \Cref{cor:squarerootcoefficients}}
\label{sec:squarerootcoefficients}
If $F(x) =  \sqrt{x}$, then its composition inverse is $G(x)=x^2$. Note that, $G^{(1)}(x)=2x, G^{(2)}(x)=2$, and $G^{(k)}(x)=0$ for all $k\geq 3$.
Since $F(1)=1$, \Cref{thm:inversebellpolynomial} gives us
\begin{align}
\begin{split}
    s_n &= G^{(1)}(F(x))\vert_{x=1} B_{n,1}(y_1, y_2, \cdots) + G^{(2)}(F(x))\vert_{x=1} B_{n,2}(y_1, y_2, \cdots) \\ &=
2 B_{n,1}(y_1, y_2, \cdots)  + 
2 B_{n,2}(y_1, y_2, \cdots)
\end{split}  
\label{eq:x_n_inverse_evaluation}
\end{align}

\noindent From the definition, $B_{n,1}(y_1, y_2, \cdots) = y_n.$  Moreover, for any $k \leq n$ (also see~\cite{comtet1974advanced}),
we have 
\[
 B_{n,k} (y_1, y_2, \cdots) = {1 \over k}\sum_{\ell = k-1}^{n-1} {n \choose \ell} y_{n - \ell} B_{\ell, k-1} (y_1, y_2, \cdots).
\]
Substituting $k=2$ yields the following result:
\[
B_{n,2}(y_1, y_2, \cdots) = {1\over 2} \sum_{\ell = 1}^{n-1} {n \choose \ell} y_{n - \ell} B_{\ell, 1} (y_1, y_2, \cdots) = {1 \over 2} \sum_{\ell=1}^{n-1} {n \choose \ell}y_\ell y_{n-\ell}.
\]
Plugging this value in \cref{eq:x_n_inverse_evaluation} concludes the proof of the theorem. 

\subsection{Existence of Power Series}
We next prove that the power series of $\sqrt{g(x)}$ exists in the neighborhood of $x=0$, where the coefficients of $g(x)$ are defined by the function $f:\mathbb N_+ \to \real_+$. 

\begin{lem}
\label{lem:neighborhood}
    Let $f:\mathbb N_+ \to \real_+$ with $f(1)=1$. Define the following polynomial: \[
g(x) = 1 + f(2)x + f(3) x^2 + f(4)x^3 + \cdots
\]
Then $\sqrt{g(x)}$ exists near the neighborhood of $x=0$ and has a formal power series. 
\end{lem}
\begin{proof}
    
 We only care about $g(x)$ in the neighborhood of $x=0$ so that we can compute the $k$-th derivative of $g(x)$ at $x=0$ to be used in our proof using Faa di Bruno's formula. Thus, it suffices to show that $\sqrt{g(x)}$ exists in the neighborhood of $x=0$. More specifically, we will show that $g(x)>0$ for all  $x \in (\eta, \infty)$ for a suitable choice of $\eta<0$. Note that for $x \ge 0$,  $g(x) \geq 1$ as all values of $f$ are positive. When $x<0$, we have 
\begin{align*}
    g(x) &= 1 + f(2)x + f(3)x^2 + \cdots 
    \geq 1 + f(2) (x+ x^3 + x^5 + \cdots) + \sum_{i=1}^\infty x^{2i} f(2i+1)
\end{align*}
as $f(\cdot)$ is non-increasing.
The summation $$\sum_{i=1}^\infty x^{2i} f(2i+1)$$ is positive.
Thus, it suffice to determine a value $\eta$ such that $ 1 + f(2) (x+ x^3 + x^5 + \cdots)  >0$ for all
$x > \eta$.
Let us define $a=f(2)$. Then, 
\[
g(x) > 1 + a (x+ x^3 + x^5 + \cdots ) = 1 + {a x \over 1-x^2},
\]
which is positive if and only if $-x^2 + xa + 1 >0$, which is true for all $x \in ( {a - \sqrt{a^2 + 4} \over 2}, {a + \sqrt{a^2 + 4} \over 2} )$. That is, for all $x \in ( {a - \sqrt{a^2 + 4} \over 2},0)$, $g(x) >0$. 
Thus $\eta = {f(2) - \sqrt{f(2)^2 + 4} \over 2}$, i.e., $g(x) > 0$ for all $x >  {f(2) - \sqrt{f(2)^2 + 4} \over 2},$ and, thus, $\sqrt{g(x)}$ exists in the neighborhood of $0$.

Now we show that $\sqrt{g(x)}$ be represented by a formal power series in the neighborhood of 0. Faa di Bruno's formula gives the coefficients of $\sqrt{g(x)}$, which contain a factor of $\frac{1}{g(x)^{1/2-k}}$. We are guaranteed by the above discussion that for $x >  {f(2) - \sqrt{f(2)^2 + 4} \over 2}$, $g(x) > 0$, and, thus, Faa di Bruno's formula is well-defined in the neighborhood of $0$, allowing us to compute $a_n$'s.

We finally argue that the power series of $\sqrt{g(x)}$ with real coefficients exists in the neighborhood of $0$ as follows. Since $g(x) = 1 + f(2) x + f(3) x^2 + \cdots,$ 
let us denote by $h(x) = f(2) x + f(3) x^2 + \cdots$.
Then the generalized Binomial formula gives us  
\[\sqrt{g(x)} = \sqrt{ 1 + (g(x)-1)} = \sqrt{1 + h(x)} = 1 + {1/2 \choose 1} h(x) + {1/2 \choose 2} h(x)^2 + \cdots \]
as long as $|h(x)| < 1$ \cite[Section 3.1]{guichard2020combinatorics}. To use this formula, we have to show that $|h(x)| <1$ when $x$ is in the neighborhood of $0$, i.e, $-1 < h(x)<1$.

\begin{claim}
    Let $h(x)$ be defined as above. Then 
    $-1 < h(x) < 1$ for  all $x \in (-{1 \over f(2)+1},
\frac{1}{f(2)+1})$.
\end{claim} 

\begin{proof}
We first claim that 
$h(x) > -1$ in the neighborhood of 0. 
We have shown above that $h(x) = g(x)-1 > -1 $ for all $x > {f(2) - \sqrt{f(2)^2 + 4} \over 2}$.

We next show that $h(x) < 1$ in the neighborhood of 0. We break the proof in two cases:
    
\begin{description}
\item[{Case $x\ge 0$}] As $f$ is non-increasing, it holds for $x \ge 0$ that 
\[
h(x) \leq f(2) x + f(2) x^2 + \cdots=  { f(2) \cdot x  \over 1 - x} 
\]
As $f(2) >0$, ${1 \over f(2)+1}<1$. Thus,
 ${ f(2) \cdot x  \over 1 - x} < 1$
for all $ x < {1 \over f(2)+1}$.

\item[{Case $x < 0$}]
It holds that $-1 < -{1 \over f(2)+1}$ and, thus,
 $x^2 < |x|$
when $-{1 \over f(2)+1} < x <0$.
Then it also follows from the fact that $f$ is non-increasing  that
\[
h(x) \leq {f(3) x^2  \over (1-x)} + \sum_{i \text{ is odd}} x^i f(i+1) \leq {f(3) x^2  \over (1-x)} \leq {f(2) |x| }.
\]
Now note that $f(2) |x| = \frac{f(2)}{f(2)+1}< 1$ for $x = -{1 \over f(2)+1}$ and $f(2) |x|$ only decreases as $x$ increases as long as $x<0$, implying that $h(x) <1$ for $-{1 \over f(2)+1} < x <0$.
\end{description}

Thus $h(x) <1$ for all $x \in (-{1 \over f(2)+1},
\frac{1}{f(2)+1})$.
\end{proof}

Since $h(x)$ is a polynomial, all powers of h(x) are polynomials, hence the generalized Binomial formula above implies that
$\sqrt{g(x)}$ has a formal power series with real coefficients for $x \in (-{1 \over f(2)+1}, {1 \over 1 + f(2)})$.  
Since $f(2) \in \real_+$ and $f(2) \leq 1$ because $f$ is a non-increasing function, we have a close neighborhood around $0$ on which the power series exist. 
\end{proof}

\subsection{Other Properties of $\gamma_2$ norm}
\begin{lem}
    For any complex matrix $C$, we have 
    \[
    \gamma_2(C) = \max {\norm{C \bullet P}_1 \over \norm{P}_1}.
    \]
\end{lem}
\begin{proof}
We use the duality of $\gamma_2(\cdot)$ norm (\Cref{thm:haagerup}) to derive the following:
\begin{align*}
    \gamma_2(C) &= \max {\norm{Q \bullet C}_2\over \norm{Q}_2}  = \max_{\norm{P}_1 \leq 1} {\tr{(C \bullet Q)P^\top} \over \norm{Q}_2} 
     = \max_{\norm{P}_1 \leq 1} {\tr{(C \bullet P)Q^\top} \over \norm{Q}_2} = \max {\norm{C \bullet P}_1 \over \norm{P}_1}
\end{align*}
This completes the proof.
\end{proof}

Schur~\cite{schur1911bemerkungen} showed the following in his seminal paper:

\begin{theorem}
[Schur~\cite{schur1911bemerkungen}]
    \label{thm:schur}
    Let $A$ and $B$ be complex matrices of the same dimension. Then
    \[
    \norm{A \bullet B} \leq \norm{A} \norm{B}.
    \]
\end{theorem}

\section{Discussion on Toeplitz Operator} 
\label{sec:toeplitz}
We use the {\em Toeplitz operator} defined over Hardy spaces. In real analysis, {\em Hardy spaces} are spaces of distributions on the real line, which are the boundary values of the holomorphic functions of the complex Hardy spaces. In complex analysis, they are spaces of holomorphic functions on the unit disk. In short, they consist of functions whose mean squared value on the unit circle remains bounded as we reach the boundary. They are natural to deal with where Lebesgue spaces are not well behaved~\cite{coifman2009extensions}). A bounded operator on a Hardy space is Toeplitz if and only if its matrix representation in the standard basis has on every diagonal the same value, i.e., the value of an entry only depends on the diagonal it belongs to. In other words, Toeplitz operators are just multiplication followed by projection onto the Hardy space. 

The theory of Toeplitz operators is vast and is covered by many communities, including operator theorists, control theorists, and statisticians. In what follows, we give a very high-level discussion on the Toeplitz operator required to understand this paper. These are standard results in operator theory and can be found in any standard textbook on Toeplitz operators~\cite{bottcher2000toeplitz,conway2000course, trefethen2005spectra}. 

We first recall that, if we are concerned with the Toeplitz matrix (i.e., finite rows and columns), then it might not be diagonalizable (let alone diagonalizable in the same basis), except for special cases, like {\em circulant} matrices. As a simple example, let $T$ be a finite positive integer and consider a $T \times T$ Toeplitz matrix  of the following form (a Jordan form of a {\em defective matrix}):
\[
A = \begin{pmatrix}
    \tau & 1 & 0 & \cdots & 0 \\ 
    0 & \tau & 1 & \cdots & 0 \\
    \vdots & \vdots & \vdots & \ddots & \vdots \\
    0 & 0  & 0 & \cdots & \tau
\end{pmatrix}
\]
for some $\tau \in \complex$. 
This matrix is a non-diagonalizable matrix because of the discrepancy between the algebraic and geometric multiplicity of eigenvalues. It is easy to verify that all its eigenvalues are $\tau$. However, the algebraic multiplicity of $\tau$ is $T$, which is greater than its {\em geometric multiplicity}, which is 1; therefore, it is not diagonalizable. 

The situation is different when we consider the Toeplitz operator and the associate symbol does not have a term $x^{-a}$ for $a \in \mathbb N_+$. For example, it is a well-known fact that the Toeplitz operator commute asymptotically~\cite{bottcher2000toeplitz}. That is, Toeplitz matrices diagonalize in the same basis when the row and column dimension tends to infinity. The diagonal entry then corresponds to the coefficients of its symbol. This is also the underlying reason behind the fact that the associated symbol of the product of two Toeplitz operators is the product of their respective symbols. Also, when all the diagonal entries in the diagonalization process are positive, then we can take either their positive square root or negative square root. This results in the square root operator consisting of only positive entries (or negative entries, respectively). This, for example, is the case when the associate symbol is $\paren{1-{x \over \alpha}}^{-1}$ and ${\mathsf{Li}_s(x) \over x}$, which using \Cref{thm:product} implies that their square root of these functions consists of either only positive coefficients or negative coefficients. 

This diagonalization also forms the basis of other lucrative properties of the Toeplitz operator that forms the basis of its wide usage. Let $\mathbb T$ denote the unit circle on the complex plane and $\complex$ denote the set of complex numbers. For functions $g : \mathbb T \to \complex$, define the infinity-norm to be
\[
\norm{g}_\infty = \sup_{\theta \in [0,2\pi]} \abs{g(e^{\iota \theta})}.
\]

Define the set $L_\infty(\mathbb T)$ to be
\[
L_\infty(\mathbb T) = \set{g: \mathbb T \to \complex : g \text{ is Lesbesgue measurable and } \norm{g} <\infty}.
\]

Suppose $g \in L_\infty(\mathbb T)$. We define the {\em multiplication operator}, $\mathcal O_g$ to be 
\[
(\mathcal O_g f)(\lambda) = g(\lambda) f(\lambda)
\]

We write $y=\mathcal O_g f$ to mean $y=gf$. It is also known that for $g \in L_\infty(\mathbb T)$, then $M_g: L_2 (\mathbb T) \to L_2(\mathbb T)$ and $\norm{M_g} = \norm{g}_\infty$. 

There is an elegant correspondence between multiplication in the frequency domain and convolution in the time domain. If $g \in \ell_2(\mathbb Z)$, then $\widehat g = Fg$, where $F$ is the Fourier transform. 
Now if $\widehat g \in L_\infty(\mathbb T)$ with its associated multiplication operator $\mathcal O_{\widehat g}$ and let $g = F^* \widehat g$, where $F^*$ is the conjugate transpose of $F$. Then 
$F^* \mathcal O_{\widehat g} F$ is a Toeplitz operator corresponding to the symbol $g$. 

The multiplication operator plays a significant role in the Toeplitz operator. For example, if $\widehat g$ is continuous, then $\mathcal O_{\widehat g}$ is invertible if and only if $\widehat g(e^{\iota \theta}) \neq 0$ for all $\theta \in [0,2\pi]$. The commutative property also follows similarly. That is $\mathcal O_{\widehat g} \mathcal O_{\widehat h} = \mathcal O_{\widehat h} \mathcal O_{\widehat g}$. 

In the theory of control theory, lower-triangular Toeplitz operators are the only operators that are {\em time-invariant} (or {\em shift-invariant}) and {\em causal}. These are important operators because  a linear time-invariant state-space system gives rise to such an operator.

\section{Some More Implications of our Bounds Applications}
\label{sec:applications}
\label{sec:operator}

\subsection{Application in Discrepancy Theory}
Discrepancy theory is an area of combinatorics in which one asks the
following question: given a finite set system $S_1,\cdots,S_m \subseteq \set{1,\cdots,u}$; color the
points $\set{1,\cdots,u}$ with two colors, say red and black, then what is the difference between red and black points in the most unbalanced set for the best coloring? In the matrix form, for an $m \times u$ matrix $A$, its discrepany is $\disc(A) = \min_{x \in\set{-1,1}^u} \norm{Ax}_\infty$. It is known that it is not a robust notion. 
Hereditary discrepancy is defined as 
\[
\herdisc(A) = \max_{
S\subseteq [m]} \disc(A\vert S).
\]
Matousek et al.~\cite{matouvsek2020factorization} showed that for any $A \in \real^{m \times u}$, the {\em hereditary discrepancy} is characterized by $\gamma_2(A)$. That is, $\herdisc(M_f) = \Omega\paren{\gamma_2(M_f) \over \log(m) }$ and  $\herdisc(M_f) = O\paren{\gamma_2(M_f) \sqrt{\log(m)} }$. Using our bound in \Cref{thm:gamma2norm}, we have a general bound for discrepancy for a large class of matrices, in particular, all matrices of the form as stated in \cref{def:Mf}.

\subsection{Application in Operator Algebra}
A matrix is called {\em partially} defined if only some of its entries are specified. One important question in operator algebra is to determine whether unspecified entries in partially defined Hermitian matrices can be filled (known as {\em completion}) to satisfy certain properties, such as {\em contraction}, {\em positive definiteness}, {\em low-rank structure}, {\em inverse eigenvalue constraints}, etc.

\subsubsection{Family of partially positive Hermitian matrix with no positive completion.}
\label{sec:applicationcompletion}
One problem that has seen a lot of interest is whether the unspecified entries of a  Hermitian matrix (that is partially positive definite) can be filled so that it is a positive definite matrix~\cite{agler1988positive,barrett1989determinantal,dym1981extensions,ellis1987invertible,grone1984positive,johnson1984inertia, johnson1990matrix,mathias1993matrix,paulsen1989schur,smith2008positive,vandenberghe2015chordal}. Recall that a Hermitian matrix $A$ is positive definite, denoted by $A \succ 0$, if all its eigenvalues are positive. Several works have given both combinatorial~\cite{grone1984positive} and algebraic characterization~\cite{paulsen1989schur} of partially positive Hermitian matrices that can be filled to make it a positive definite matrix. These characterizations also give us a method to construct partially positive Hermitian matrices that cannot be completed to be positive definite matrices. To the best of our knowledge, all these constructions are combinatorial and are matrix representations of {\em chordal graphs}. 

We use \cite[Lemma 3.1]{paulsen1989schur} that implies that a matrix $P = \begin{pmatrix}
    Q & A \\
    A^\top & Q
\end{pmatrix}$, with $|Q[i,i]| = 1$ (and other entries unspecified), is partially positive (i.e., they are symmetric and every principal specified submatrix is positive) as long as all the specified $(i,j)$-th entries of $A$ satisfy $|A[i,j]|\leq 1$  and $P$ has a positive completion only if $\norm{A \bullet X} \leq 1$ for all $X$, where $\norm{\cdot}$ denote the spectral norm (also see \cite[Remark 1]{paulsen1989schur}). Paulsen et al.~\cite[page 162]{paulsen1989schur} used the result of Kwapien and Pe{\l}czy{\'n}ski~\cite{kwapien1970main}\footnote{This was later improved to be more precise by Mathias~\cite{mathias1993hadamard} Basically, Mathias~\cite{mathias1993hadamard} showed that, if $f(n)=1$, then for $M_f \in \set{0,1}^{T \times T}$, ${1 \over 2T} + {\log(T) \over \pi} \leq \gamma_2(M_f) \leq {1\over 2} + {\log(T) \over \pi}$.}  which states that $\gamma_2(M_f)$ with $M_f \in \set{0,1}^{T \times T}$ and $f(n)=1$ for all $n \in \mathbb N_+$ is $\Theta(\log(T))$ along with \cite[Lemma 3.1]{paulsen1989schur} to show that the partially completed Hermitian matrix $P= \begin{pmatrix}
    Q & A \\
    A^\top & Q
\end{pmatrix}$ with $A = M_f$ for $f(n)=1$  does not have a positive completion.

Our result allows us to extend their argument to a more general class of lower triangular matrices. Schur's bound (\Cref{thm:schur}) and the dual characterization of $\gamma_2(A)$ (\cref{eq:dual}) implies that $\gamma_2(X) \leq \norm{X}$ for any matrix $X$. So, \Cref{thm:gamma2norm} implies $\norm{M_f} \geq \gamma_2(M_f) >1$. Since $M_f[i,i]\leq 1$, this implies that the following  infinite family of Hermitian matrices 
\[
P= \begin{pmatrix}
    Q & M_f \\
    M_f^\top & Q
\end{pmatrix},
\]
where $M_f$ is as defined in \cref{def:Mf} satisfies the following claim: $P$ is partially positive definite, but the unspecified entries of $Q$ cannot be instantiated to ensure that $P \succ 0$.

\subsubsection{Non-existence of contraction map for a large class of matrices.} Our lower bound also implies a negative result with respect to contraction maps, which follows immediately from \cite[Proposition 3.1]{paulsen1989schur}. Since this implication is a straightforward application of our lower bound as in \Cref{sec:applicationcompletion}, we do not expand more on it and leave it as an easy exercise.

\section{Some plots comparing exact $a_n$ and our estimates in \Cref{cor:gamma2norm}}
\label{sec:figures}
To get a sense of how close our estimate of coefficients are to the exact values $a_n$ using the evaluation of Bell's polynomial, we compute the values of $a_n$ using \Cref{rem:coefficient} and then compare it with the estimate computed in the proof of \Cref{cor:gamma2norm}. We plot the gap in \Cref{fig:coeff_gap}. We consider two functions ${\mathsf{Li}_1(x) \over x}$ (i.e., $c=1$) and ${\mathsf{Li}_2(x) \over x}$ (i.e., $c=2$). The $x$-axis in \Cref{fig:coeff_gap} is the index of the coefficient of the $n$-th term of the square root from $n\geq 2$ and on the $y$-axis is the gap between the estimate we compute in the proof of \Cref{cor:gamma2norm} and that of the exact coefficients in \Cref{thm:gamma2norm} that relies on the evaluation of Bell's polynomial.  

\begin{figure}[h]
    \centering
    \includegraphics[scale=0.5]{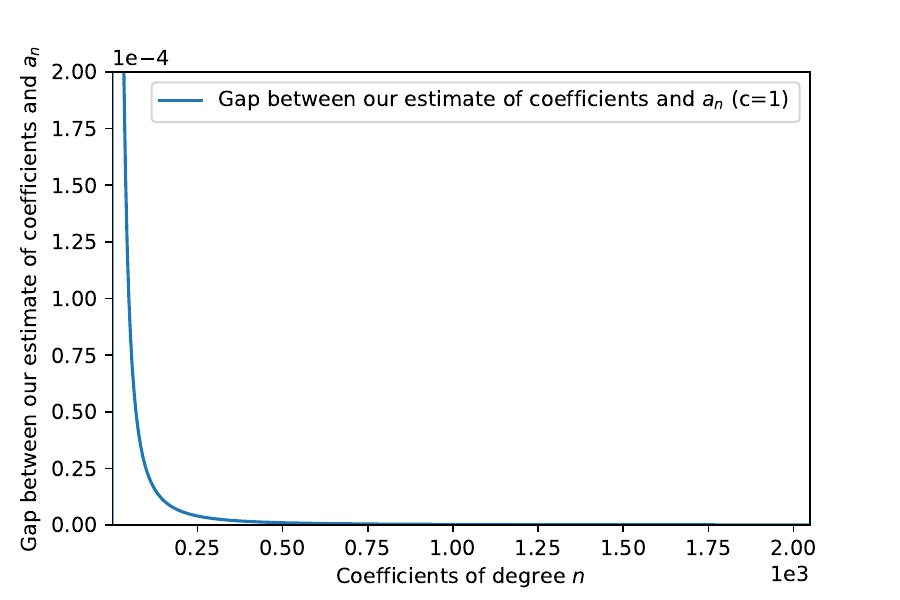}
    \includegraphics[scale=0.5]{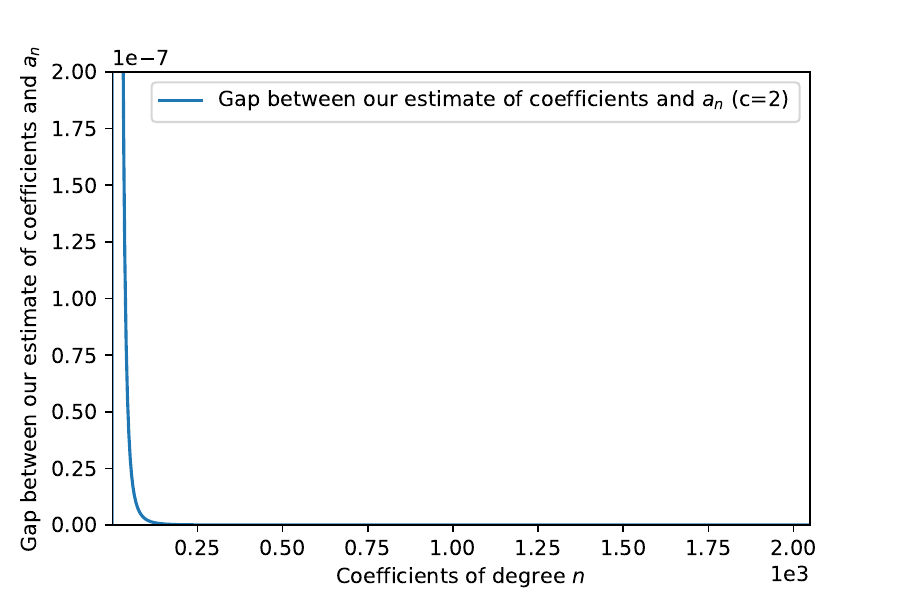}
    \caption{Gap between our estimates of the coefficients of $\sqrt{\sum_{i\geq 0} f(i+1)x^i}$ with $f(m)=m^{-c}$ for $c=\set{1,2}$ and the exact coefficients for the first $2048$ coefficients.}
    \label{fig:coeff_gap}
\end{figure}

For $c=1$, the gap reduces to $5.960464477539063 \times 10^{-8}$ and for  $c=2$, the gap reduces to $1.4210854715202004 \times 10^{-14}$. This can be seen in the magnified plot that only focuses on the coefficients $a_{1024}$ to $a_{2048}$ in \Cref{fig:coeff_gap_magnified}.

\begin{figure}[h]
    \centering
    \includegraphics[scale=0.5]{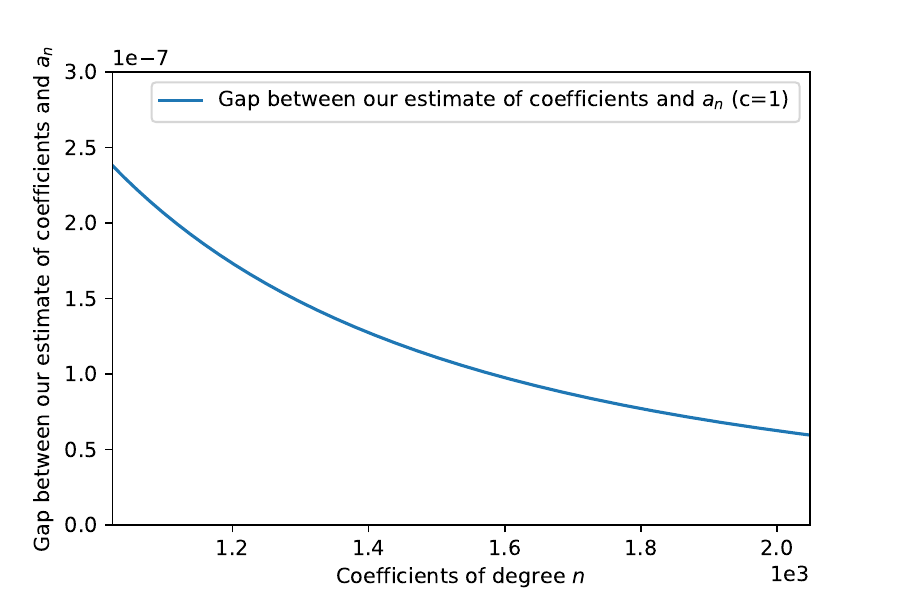}
    \includegraphics[scale=0.5]{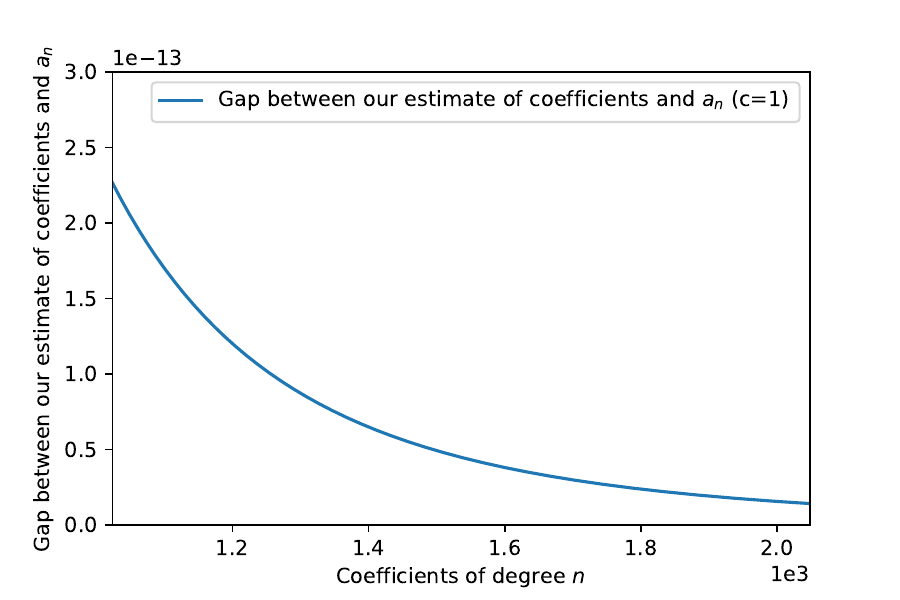}
    \caption{Gap between our estimates of the coefficients of $\sqrt{\sum_{i\geq 0} f(i+1)x^i}$ with $f(n)=n^{-c}$ for $c=\set{1,2}$ and the exact coefficients for the first $2048$ coefficients.}
    \label{fig:coeff_gap_magnified}
\end{figure}

\end{document}